\documentclass[journal]{IEEEtran}
\usepackage{amsmath,amssymb,amsfonts}
\usepackage{amsthm}
\usepackage{calligra}
\usepackage{algorithm}
\usepackage{algpseudocode}
\usepackage{bm}
\usepackage{cite}
\usepackage{hyperref}
\usepackage{enumerate}
\usepackage[table]{xcolor}
\usepackage{graphicx}
\usepackage{subfigure}
\usepackage{epstopdf}
\usepackage{cite,url,epsfig}
\usepackage{epstopdf}
\usepackage{caption}
\usepackage{comment}
\usepackage{booktabs}

\newtheorem{definition}{Definition}
\newtheorem{theorem}{Theorem}
\newtheorem{lemma}{Lemma}

\newtheorem{remark}{Remark}

\begin{document}
\title{Generative Diffusion-based Contract Design for Efficient AI Twins Migration in Vehicular Embodied AI Networks}

\author{Yue Zhong, Jiawen Kang*, Jinbo Wen, Dongdong Ye, Jiangtian Nie, Dusit Niyato,~\IEEEmembership{Fellow,~IEEE,} \\ Xiaozheng Gao, Shengli Xie,~\IEEEmembership{Fellow,~IEEE}

\thanks{ 
Y. Zhong, J. Kang, D. Ye, and S. Xie are with the School of Automation, Guangdong University of Technology, Guangzhou 510006, China (e-mail: 2112404106@mail2.gdut.edu.cn; kavinkang@gdut.edu.cn; dongdongye8@163.com; shlxie@gdut.edu.cn). 

J. Wen is with the School of Computer Science and Technology, Nanjing University of Aeronautics and Astronautics, Nanjing 210016, China (e-mail: jinbo1608@163.com). 


J. Nie and D. Niyato are with the School of Computer Science and Engineering, Nanyang Technological University, Singapore (e-mail: jnie001@e.ntu.edu.sg; DNIYATO@ntu.edu.sg). 

X. Gao is with the School of Information and  Electronics, Beijing Institute of Technology, Beijing 100081, China (e-mail: gaoxiaozheng@bit.edu.cn).

(\textit{*Corresponding author: Jiawen Kang})



	} 
}

\maketitle

\begin{abstract}
Embodied Artificial Intelligence (AI) is a rapidly advancing field that bridges the gap between cyberspace and physical space, enabling a wide range of applications. This evolution has led to the development of the \underline{\textbf{V}}ehicular \underline{\textbf{E}}mbodied \underline{\textbf{A}}I \underline{\textbf{NET}}work (VEANET), where advanced AI capabilities are integrated into vehicular systems to enhance autonomous operations and decision-making. Embodied agents, such as Autonomous Vehicles (AVs), are autonomous entities that can perceive their environment and take actions to achieve specific goals, actively interacting with the physical world. Embodied twins are digital models of these embodied agents, with various embodied AI twins for intelligent applications in cyberspace. In VEANET, embodied AI twins act as in-vehicle AI assistants to perform diverse tasks supporting autonomous driving using generative AI models. Due to limited computational resources of AVs, these AVs often offload computationally intensive tasks, such as constructing and updating embodied AI twins, to nearby RoadSide Units (RSUs). However, since the rapid mobility of AVs and the limited provision coverage of a single RSU, embodied AI twins require dynamic migrations from current RSU to other RSUs in real-time, resulting in the challenge of selecting suitable RSUs for efficient embodied AI twins migrations. Given information asymmetry, AVs cannot know the detailed information of RSUs. To this end, in this paper, we construct a multi-dimensional contract theoretical model between AVs and alternative RSUs. Considering that AVs may exhibit irrational behavior, we utilize prospect theory instead of expected utility theory to model the actual utilities of AVs. Finally, we employ a generative diffusion model-based algorithm to identify the optimal contract designs, thus enhancing the efficiency of embodied AI twins migrations. Compared with traditional deep reinforcement learning algorithms, numerical results demonstrate the effectiveness of the proposed scheme.
\end{abstract} 

\IEEEpeerreviewmaketitle

\section{Introduction}\label{Intro}
Embodied Artificial Intelligence (AI) refers to autonomous systems or robots that demonstrate intelligent behaviors within the physical environment by interacting with their surroundings through their bodies, finding applications across various fields \cite{duan2022survey}. The advancement of embodied AI in vehicular systems has led to the development of \underline{\textbf{V}}ehicular \underline{\textbf{E}}mbodied \underline{\textbf{A}}I \underline{\textbf{NET}}works (VEANETs), where vehicles integrate sensory input and motor capabilities to achieve real-time contextual awareness and adaptive decision-making \cite{cunneen2019autonomous}. Embodied agents, including Autonomous Vehicles (AVs) within VEANETs, actively perceive and interact with both virtual and physical environments, enabling them to understand human intentions, decompose complex tasks, and interact effectively with their surroundings \cite{liu2024aligning}. In VEANETs, embodied twins and AI twins have been proposed to enhance the intelligence of networks by integrating Digital Twins (DTs) into embodied AI systems. With an embodied world model serving as the ``brain" of agents, embodied twins facilitate the transfer of skills from virtual to real-world scenarios \cite{wu2024embodied}. Similar to the concept of DTs \cite{10415196}, embodied twins refer to digital models created through real-time data analytics and simulation, representing the complete life cycle of embodied agents in the virtual environment, encompassing multiple embodied AI twins. Specifically, vehicular embodied AI twins serve as in-vehicle AI assistants performing various tasks in AVs, and require constant updates to ensure real-time synchronization between physical and virtual spaces \cite{xu2023epvisa}. These software entities autonomously perform functions within their transportation environment, enabling independent cognition, decision-making and action without drivers, effectively simulating real-world decision-making processes. Depending on diverse demands, AVs equipped with embodied AI twins can offer various services to passengers, such as Augmented Reality (AR) navigation and Intelligent Cruise Control (ICC) \cite{greguric2024impact}, thereby providing an interactive and immersive experience for users within the vehicle. 


Considering the limited resources of an AV, these intensive computational tasks of embodied AI twins need to be offloaded to edge servers in the nearby RoadSide Units (RSUs) with more communicational and computing resources \cite{liang2019efficient}. However, as the AV moves, the AV may leave the current RSU with limited provision coverage. Thus it is hard to guarantee the continuity of in-vehicle services in the AV when its embodied AI twins are still in the current RSU \cite{chen2023multi}. Therefore, the embodied AI twins must undergo real-time migration from the current RSU to a new RSU to guarantee seamless delivery of in-vehicle services to the AV. This necessitates the development of an incentive framework aimed at encouraging the maximal participation of RSUs in provisioning required resources for embodied AI twins.

Given the uncertainty in interactions between RSUs and AVs, AVs may exhibit irrational behavior. Managers often rely on cognitive biases such as regret aversion, confirmation bias, and recency bias, when making complex decisions under time constraints and incomplete information, leading to suboptimal choices \cite{aren2021biases}. For instance, AVs might favor RSUs that previously provided favorable data, even if less reliable, neglecting more accurate RSUs and resulting in inefficient traffic management and potential safety risks. Consequently, the application of Expected Utility Theory (EUT) to establish the utility of AVs is not considered rational. In light of this, the authors in \cite{eec14168-5714-3ca8-b073-d038266f2734} introduced a novel model of risk attitudes known as Prospect Theory (PT). This model effectively captures empirical evidence of risk-taking behavior, including observed deviations from EUT. There have been studies integrating PT into the construction of utility functions in contract theory to better capture the subjective utility of users in the model \cite{10254627,huang2021efficient,9739801}. As a result, leveraging PT allows us to incorporate the subjective utility of AVs, resulting in a more accurate and meaningful model. 

To tackle the aforementioned challenges, we consider an incentive mechanism utilizing PT for efficient embodied AI twins migration. In this regard, we introduce a contract model designed to incentivize RSUs to provide resources for service provision of embodied AI twins. Recognizing the inherent uncertainty experienced by AVs in uncertain environments, we formulate a novel contract model by incorporating PT. The new contract model facilitates the establishment of a subjective utility function for AVs, which considers their preferences and decision-making processes. Moreover, Generative Diffusion Models (GDMs) present a promising tool for resolving optimization. Therefore, we employ a GDM-based scheme to determine optimal contracts. The main contributions of this paper are summarized as follows:

\begin{itemize}
    \item \textit{To the best of our knowledge, this is the first work to propose the concept of ``embodied twins" and ``embodied AI twins".} Embodied twins are digital counterparts of embodied agents within virtual environments, while embodied AI twins are integral elements of these embodied twins, which refer to replicas created by AI algorithms to execute diverse sub-tasks or functions of the embodied agents. AVs act as embodied agents within VEANETs, with their embodied AI twins acting as in-vehicle AI assistants that offer diverse services to passengers.
    \item In VEANETs, considering that AVs lack specific information about the resources and capabilities of RSUs, we apply the contract theory to address this information asymmetry. We develop a multi-dimensional contract model where the AV acts as the contract designer and RSUs act as the contract selectors. To better measure the perception capability of the embodied AI twins, we incorporate virtual immersion metrics of users within AVs into the utility function of AVs.
    \item We integrate multi-dimensional contract theory with PT to design an incentive mechanism that effectively encourages AVs to participate in embodied AI twins migration. By utilizing the framing effect in PT, we capture the risk-aware behavior of AVs, thereby enhancing the acceptability of the incentive mechanism in practical applications. Subsequently, we analyze the solution for multi-dimensional contract design.
    \item We employ a GDM-based algorithm to find the optimal contract designs, leveraging the  GDM to address the high dimensionality and complexity of the formulated problem. Through subsequent numerical analysis, we demonstrate that the proposed GDM-based scheme outperforms the traditional Deep Reinforcement Learning (DRL)-based scheme in terms of efficiency.
\end{itemize}

The rest of this paper is organized as follows. In Section \ref{RW}, we review the related work. In Section \ref{framework}, we propose the overall framework of this paper and introduce the preliminaries of the PT and GDM. In Section \ref{Problem}, we present the problem formulation, propose the contract model, demonstrate the contract feasibility, and integrate the PT into the incentive mechanism. In Section \ref{Optimial_Contract}, we propose a GDM-based algorithm to find the optimal contract designs. Section \ref{Results} shows numerical results about our proposed model. Finally, Section \ref{Conclusion} summarizes the paper.

\section{Related Work}\label{RW}

\subsection{Vehicular Embodied AI Networks}
\begin{figure}[t]
\centering
\includegraphics[width=0.45\textwidth]{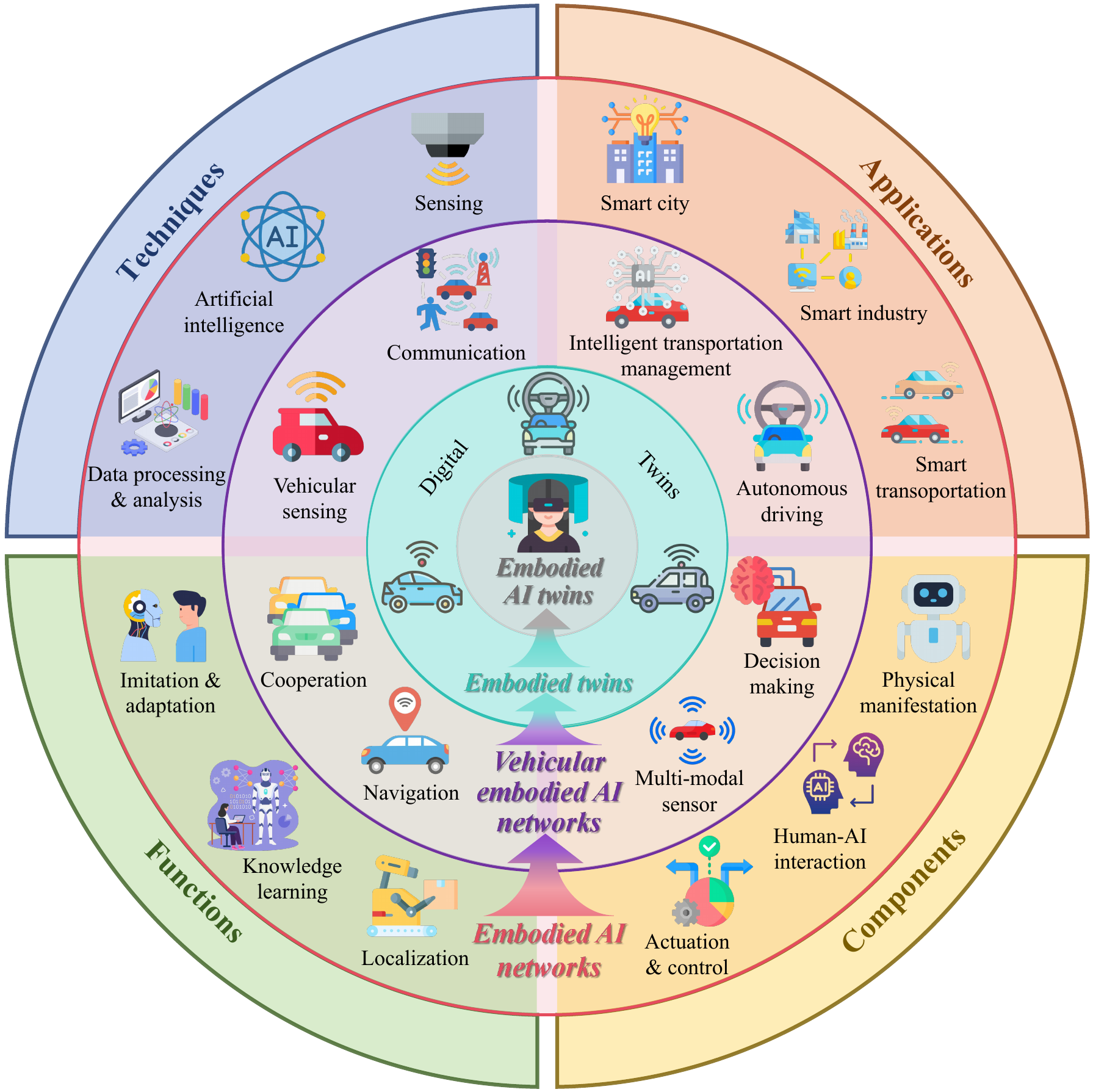}
\captionsetup{font=footnotesize}
\caption{The illustration of the techniques, components, functions, and applications of Embodied AI networks and VEAINETs.}
\label{VEANET}
\end{figure}
The concept of ``Embodied AI” was first derived by Alan Turing from his paper named ``Computing Machinery and Intelligence", published in 1950, which introduced the idea now widely known as the Turing Test \cite{turing2009computing}. Turing posed the question of whether machines can think, exploring the possibility of creating agents that exhibit intelligence not only in solving abstract problems in cyberspace but also in performing complex tasks in the physical world. The rapid advancement of embodied AI has expanded its applications across various fields, as shown in Fig. \ref{VEANET}, garnering significant attention from the research community \cite{schmalzried2024role}. The integration of embodied AI with vehicular networks has led to the emergence of a new paradigm known as VEANETs, where AVs play a pivotal role \cite{cunneen2019autonomous}. By leveraging multi-modal perception and coordinated actions, VEANETs enhance vehicular intelligence, allowing for autonomous navigation and interaction within dynamic, unpredictable environments.


Embodied agents are at the core of embodied AI, functioning as intelligent entities that interact with the physical world. The development of Generative AI (GAI) models, e.g., Large Language Models (LLMs), Vision Language Models (VLMs), and Vision Language Action (VLA) models, has significantly enhanced the perception, interaction, and planning capabilities of foundational models \cite{londono2024fairness}. These developments have enabled the creation of versatile embodied agents capable of seamless interaction in both virtual and physical environments, making them an ideal platform for deploying Multi-modal Large Models (MLMs) \cite{huang2022inner}. Embodied agents are equipped with multi-modal sensors, e.g., cameras, microphones and tactile sensors, enabling them to perceive and interact with their surroundings in real-time \cite{wang2024multimodal}. Additionally, they often feature actuators, e.g., robotic arms, wheels or legs, which allow them to physically engage with objects and navigate their environment effectively. Their cognitive abilities enable them to comprehend and operate in complex real-world environments, making real-time decisions in dynamic and unpredictable situations without constant human oversight \cite{liu2024aligning}. AVs equipped with AI-powered sensors and advanced algorithms, exemplify embodied agents within VEANETs, achieving human-like perception and decision-making capabilities \cite{bathla2022autonomous}.


DTs are virtual counterparts that faithfully represent the complete life cycle of physical objects within a virtual environment \cite{10415196}. Similarly, embodied twins are digital representations of embodied agents, with embodied AI twins specifically being digital replicas created using AI algorithms to perform sub-tasks or functions of these agents. These AI twins can extend the capabilities they have developed in virtual environments into the real world. In VEANETs, embodied twins represent AVs as digital models within the digital environment, with embodied AI twins serving as in-vehicle virtual assistants, supporting AI-driven services like AR navigation and ICC. Since different in-vehicle services require distinct resource allocations, a single AV offering multiple services may need varied resource distributions from RSUs \cite{wen2023task}. The dynamic vehicular physical world contains essential information and attributes of tangible entities, necessitating continuous updates to the real-world characteristics of embodied AI twins in the virtual realm \cite{du2023yolo}. The interplay between vehicular movement and limited RSU coverage poses challenges, requiring real-time migration of embodied AI twins between RSUs \cite{zhong2023blockchain}. This process entails transitioning from RSUs currently providing resources to RSUs on the verge of assuming coverage responsibility.

\subsection{Incentive Mechanisms for Twins Migration}
Establishing the virtual space and providing in-vehicle services entail substantial resource consumption, particularly in terms of computing resources required to handle the intensive data, extensive storage resources, and robust network resources necessary to maintain ultra-high-speed and low-latency connections \cite{10158923}. Therefore, it is imperative to tackle the challenges associated with resource allocation and devise an incentive mechanism that encourages virtual service providers to offer their resources \cite{9838422, 9838736, 10302973}. In \cite{wen2023task}, the authors introduced an incentive mechanism for migrating Vehicle Twins (VTs) within the virtual space, addressing the challenge of ensuring uninterrupted services despite limited RSU coverage and vehicle mobility. They proposed an Age of Migration Task (AoMT) metric to measure task freshness and an AoMT-based contract model to incentivize RSUs to contribute sufficient bandwidth resources. In \cite{zhong2023blockchain}, a blockchain-assisted game approach framework was introduced for ensuring reliable VT migration within vehicular metaverses. The authors devised a single-leader multi-followers Stackelberg game involving a chosen coalition of RSU and Vehicular Metaverse Users (VMUs) to incentivize VMUs to engage in VT migrations. In \cite{10302973}, the authors introduced a learning-based incentive mechanism, i.e., the Stackelberg model, for VT migration in vehicular metaverses, addressing the challenge of ensuring seamless experiences for users within vehicles amidst limited RSU coverage and mobility. 

Recent research has begun addressing the resource optimization challenges related to digital twins migration caused by vehicle movement in vehicular metaverses, along with the development of incentive mechanisms. However, these studies remain relatively narrow in scope and do not extend to twins migration within VEANETs. In \cite{isprs2021}, the authors explored the intersection of environmental sensing, immersive technologies, and embodied cognition to lay the groundwork for embodied digital twins. They proposed leveraging theoretical foundations of embodied cognition to develop research frameworks for advancing the concept, emphasizing the conversion of environmental data into immersive experiences. Although the authors in this paper proposed the concept of ``embodied digital twins", they did not consider the problem of twins migration. The authors in \cite{liu2024aligning} examined nearly 400 papers, initially presenting a selection of prominent embodied robots and embodied simulation platforms. Subsequently, it delved into discussions on embodied perception, embodied interaction, embodied intelligent bodies, and virtual-to-real migration. However, it did not cover pertinent literature on the migration of twins in VEANETs. Consequently, developing incentive mechanisms for twins migration in VEANETs is crucial for advancing this field.

\section{System Model}\label{framework}
In this section, we introduce the incentive mechanism framework proposed in this paper, as well as the preliminary concepts of PT and GDMs.
\begin{figure*}[t]
\centering
\includegraphics[width=0.95\textwidth]{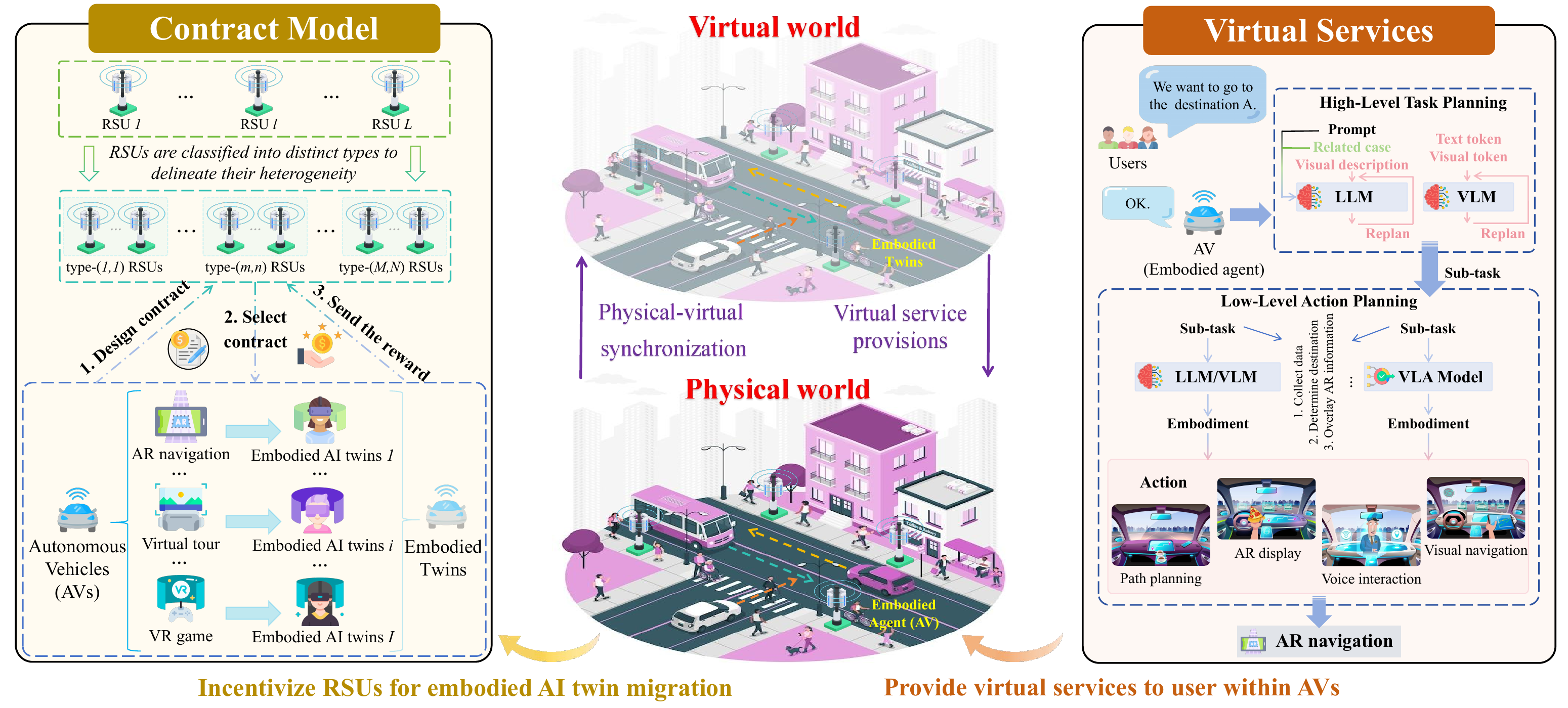}
\captionsetup{font=footnotesize}
\caption{The left part is the multi-dimensional contract-based embodied AI twins migration framework in VEANETs. The right part is the schematic diagram of embodied agents completing tasks, consisting of a high-level task planning module and a low-level action planning module.}
\label{system_fig}
\end{figure*}
\subsection{Multi-dimensional Contract-based Embodied AI twins migration Framework}

AVs continuously generate and execute computationally intensive embodied AI twins tasks to ensure this synchronization within the virtual space \cite{xu2023generative}. However, due to limitations in local resources, AVs may be unable to handle these tasks and update the tasks. To address this, AVs delegate the execution of computationally intensive embodied AI twins tasks to RSUs equipped with robust computing and communication infrastructure \cite{alkhoori2024latency}. By offloading these tasks to RSUs, AVs ensure real-time execution and seamless synchronization with the virtual space. Additionally, RSUs can utilize information from the embodied AI twins of AVs to assist in service provision for users within the AVs. Due to the limited service coverage of RSUs and the mobility of AVs, the embodied AI twins need to be migrated from the current RSUs to the next RSUs \cite{zhong2023blockchain}. We consider hotspot areas, e.g., intersections and areas near commercial streets, there are multiple RSUs in the area, and AVs need to decide the target RSUs to migrate their embodied AI twins to based on the diverse requirements of in-vehicle services. In the interaction between RSUs and AVs, AVs compensate RSUs for services rendered by paying rewards. In contrast, RSUs fulfill their role by provisioning the necessary resources for executing embodied AI twins tasks. Therefore, we introduce an incentive mechanism framework between RSUs and AVs, incentivizing RSUs to offer resources for embodied AI twins migration. The multi-dimensional contract-based embodied AI twins migration framework and the steps for AVs to perform tasks are shown in Fig. \ref{system_fig}, and the detailed information is described as follows.

\textit{\textbf{Step 1: Send embodied AI twins migration requests to the current RSUs}}: As shown in the middle part of Fig. \ref{system_fig}, when AVs are in motion on the road, continuous in-vehicle services cannot be provided to users within the AVs due to the limited service coverage of RSUs. To ensure a seamless immersive experience with AVs, the embodied AI twins should be migrated from current RSUs to other RSUs \cite{zhong2023blockchain, wen2023task}. To initiate this migration process, AVs send embodied AI twins migration requests to current RSUs. Subsequently, RSUs broadcast their requests to surrounding RSUs, facilitating the seamless transfer of embodied AI twins and the uninterrupted delivery of in-vehicle services to AVs.

\textit{\textbf{Step 2: Construct a multi-dimensional contract model for embodied AI twins migration between AVs and RSUs}}: As shown in the left part of Fig. \ref{system_fig}, to address the information asymmetry between AVs and RSUs and incentivize RSUs to allocate computing and bandwidth resources for embodied AI twins migration, a contract model is developed, which contains three steps. 1) AVs design multi-dimensional contracts for all types of RSUs, i.e., AVs serve as contract designers, determining the terms of contracts tailored for individual RSUs; 2) RSUs select the contract designed for themselves, i.e., RSUs act as contract choosers, selecting the optimal contract offered by AVs; 3) AVs send the reward to RSUs, i.e., RSUs provide resources for AVs based on the selected contracts and receive the corresponding rewards. This contractual arrangement ensures that both parties are aligned in their objectives, promoting cooperation and resource allocation efficiency during the transaction process \cite{10302973}. Furthermore, from the left part, we observe that the embodied twins represent the virtual model of AVs, with each embodied twins containing multiple embodied AI twins. We assume that there are $I$ embodied AI twins within an AV. Denote the embodied twins as $E_T$ and the $i$-th embodied AI twins as $E_{AI_T}^i$ for $0\le i\le I$. Since each embodied twins contains several embodied AI twins, we can express the embodied twins as $E_T=\{E_{AI_T}^1,\cdots, E_{AI_T}^i,\cdots, E_{AI_T}^I\}$.

\textit{\textbf{Step 3: Receive resources from target RSUs and provide in-vehicle services to users}}: Once RSUs select the optimal contract, they allocate the designated resources to the embodied AI twins task and receive the corresponding reward. The right side of Fig. \ref{system_fig} illustrates how embodied agents (i.e., AVs), undertake tasks like AR navigation. To accomplish these tasks, embodied agents typically follow these processes \cite{liu2024aligning}: 1) High-level embodied task planning: This process involves breaking down abstract and intricate tasks into specific sub-tasks; 2) Low-level embodied action planning: The agents incrementally execute these sub-tasks by utilizing embodied perception and interaction models. LLMs and VLMs play crucial roles in facilitating embodied task planning. Embodied agents can approach action planning through two strategies: using pre-trained perception and intervention models as tools to systematically complete sub-tasks or by directly deriving action planning from the capabilities of the VLA model \cite{liu2024aligning}. Upon completing the action planning, the embodied AI twins migration is completed, and the embodied AI twins can continue to request resources from RSUs to ensure task execution, enabling AVs to provide seamless in-vehicle services to users.

\subsection{Prospect Theory}
In 1979, two prominent Israeli psychologists, Daniel Kahneman and Amos Tversky, made a significant contribution to the field of decision-making under risk with their paper titled ``Prospect Theory: An Analysis of Decision-Making under Risk",  published in the journal Econometrics \cite{eec14168-5714-3ca8-b073-d038266f2734}. The proposed framework provides valuable insights into the intricacies of decision-making under uncertainty and risk, thus highlighting the limitations of traditional utility-based theories (e.g., EUT) and providing a comprehensive analysis of decision-making in uncertain scenarios. There are two main differences between PT and EUT. 
\begin{enumerate}[1.]
    \item PT integrates subjective probabilities to ascertain the weighting allocated to each potential outcome. Subjective probability is derived from objective probability \cite{9739801}. 
    \item Decision-makers employ reference points based on specific objectives to classify outcome returns as either gains or losses in PT. Falling short of this goal is perceived as a loss while exceeding it is deemed a gain \cite{10254627}. 
\end{enumerate}

We derive the utility function form of PT in the following. We consider a system with $k$ AVs, denoted by the set $\mathcal{K}=\{1,\cdots,k,\cdots,K\}$. The utility function for all AVs, based on EUT, is defined as
\begin{equation}
    U^{EUT}=\sum_{k=1}^K P_k U_{k}^{EUT},
\end{equation} 
where $P_k$ represents the objective probability, and $U_{k}^{EUT}$ denotes the utility of the AV $k$. In uncertain and risky environments, AVs may exhibit irrational behavior. To address this, we can leverage the fundamental principles of PT to construct a utility function that captures their decision-making process more effectively. The utility function based on PT can be expressed as \cite{huang2021efficient}
\begin{equation}
    U^{PT}=\sum_{k=1}^K H(P_k) U_{k}^{PT},  
\end{equation}
where $H(P_k)=\exp(-(-\log(P_k))^\alpha)$ represents the inverse S-shape probability weighting function applied to the objective probability $P_{m,n}$. This weighting function introduces a psychological bias, characterized by an underestimation of high-probability events and an overestimation of low-probability events \cite{9739801}. The rational coefficient $\alpha$ is employed to quantify the extent of distortion in the subjective evaluation of objective probabilities, thereby influencing the overall shape of the weighting function \cite{eec14168-5714-3ca8-b073-d038266f2734}. Consequently, $U_{k}^{PT}$ can be calculated as
\begin{equation} \label{PT}
U_{k}^{PT}= \left\{ \begin{aligned}
(U_{k}^{EUT}-U_{k}^{ref})^{\eta^{+}},\: U_{k}^{EUT}\geq U_{k}^{ref},\\
- \nu (U_{k}^{ref}-U_{k}^{EUT})^{\eta^{-}},\: U_{k}^{EUT}<U_{k}^{ref},\\
\end{aligned} \right.
\end{equation}
where $\eta^{+}, \eta^{-}\in (0,1]$ serve as weighting factors that capture the distortion of gains and losses, respectively. $\nu\geq 0$  reflects the level of loss aversion. The reference point $U_{k}^{ref}$ is introduced to classify the utility $U_{k}^{EUT}$ as either a gain or a loss, further enhancing the applicability of the PT framework \cite{10254627}.

\subsection{Generative Diffusion Models}
The advent of GAI presents transformative potential extending beyond conventional AI paradigms. Unlike conventional AI frameworks predominantly oriented towards the analysis or classification of pre-existing data, GAI possesses the capability to generate novel datasets encompassing various modalities such as textual, visual, auditory, and synthetic temporal sequences, among others \cite{9903869}. GAI encompasses a diverse array of models and methodologies, e.g., Transformer, Generative Adversarial Networks (GANs), and GDMs, these models and methodologies possess distinct advantages and applications within the realm of AI \cite{du2023deep}. Their contributions to the progression of AI exhibit variations, with GDMs standing out as particularly influential in this context, primarily owing to their distinctive methodology for data generation and their aptitude for modeling intricate data distributions \cite{10419041}. GDMs employ a progressive forward diffusion process based on the initial input data, gradually introducing Gaussian noise. Then, GDMs employ a reverse diffusion process through a denoising network, which iteratively approximates real samples represented as $x \sim q(x)$ through a series of estimation steps, and $q(x)$ represents the underlying data distribution \cite{du2023deep,yang2023diffusion}. Subsequently, the denoising network undergoes training to reverse the noise process and restore both the data and its content, thereby facilitating novel data generation. The following describes the forward and reverse diffusion process in further detail:

\subsubsection{Forward diffusion process} 
Considering a given data distribution $x_0 \sim q(x_0)$, the forward process in GDMs can be accurately represented as a Markov process comprising $T$ steps. Gaussian noise is applied to the initial sample $x_0$ in the forward diffusion process, resulting in the generation of a series of samples $\{x_1,x_2,\cdots,x_T\}$ \cite{10172151}. This progression is governed by the transition kernel $q(x_t|x_{t-1})$, which captures the dynamics of the system \cite{ho2020denoising}. By utilizing the chain rule of probability and leveraging the Markov property, the joint distribution of $\{x_1,x_2,\cdots,x_T\}$ conditioned on $x_0$ can be decomposed as $\prod_{t=1}^T q(x_t|x_{t-1})$, i.e.,
\begin{equation}
    \begin{split}
        q(x_1,&x_2,\cdots,x_T|x_0)=\prod_{t=1}^T q(x_t|x_{t-1}),\\
        q(x_t|x_{t-1})&=\mathcal{N}(x_t;\boldsymbol\mu_t=\sqrt{1-\iota_t}x_{t-1},\boldsymbol\Sigma_t=\iota_t\textbf{I}),
    \end{split}
\end{equation}
where $\boldsymbol\mu_t$ and $\boldsymbol{\Sigma}_t$ denote the mean and variance, respectively, of the normal distribution at step $t$. $\textbf{I}$ represents that each dimension has the same standard deviation $\iota_t$ and is the identity matrix. To simplify the expression, we define $\lambda_t:=1-\iota_t$ and $\hat{\lambda}_t:=\prod_{i=0}^t\iota_i$. 
Given the input content $x_0$, sampling the Gaussian vector $\boldsymbol{\epsilon} \sim \mathcal N (\textbf{0}, \textbf{I})$, $x_t$ can be obtained by \cite{du2023deep}
\begin{equation}
    x_t=\sqrt{\hat{\lambda}_t}x_0+\sqrt{(1-\hat\lambda_t)}\boldsymbol{\epsilon}_0,
\end{equation}
Therefore, $x_t$ can be obtained by the following distribution
\begin{equation}
    x_t\sim q(x_t|x_{t-1})=\mathcal{N}(x_t;\sqrt{\hat{\lambda}_t}x_0,(1-\lambda_t)\textbf{I}).
\end{equation}
    
\subsubsection{Reverse diffusion process} 
Based on the inverse distribution $q(x_{t-1}|x_t)$, it becomes feasible to sample $x_t$ from the standard normal distribution $\mathcal{N}(\textbf{0},\textbf{I})$ using a reverse process. A crucial factor contributing to the effectiveness of this sampling process is the training of the reverse Markov chain to accurately replicate the time reversal of the forward Markov chain \cite{yang2023diffusion}. Nevertheless, accurately estimating the statistical properties of $q(x_{t-1}|x_t)$ necessitates intricate computations involving the data distribution, which poses a formidable challenge. To address this challenge, a parametric model $p_\theta$ can be employed to approximate the estimation of $q(x_{t-1}|x_t)$ as follows, which is given by \cite{ho2020denoising}
    \begin{equation}\label{p}
        p_\theta(x_{t-1}|x_t)=\mathcal{N}(x_{t-1};\boldsymbol\mu_\theta(x_t,t),\boldsymbol\Sigma_\theta(x_t,t)),
    \end{equation}
where $\theta$ represents the model parameters. Thus, the trajectory from $x_T$ to $x_0$ is expressed as \cite{du2023deep}
    \begin{equation}
        p_\theta(x_0,x_1,\cdots,x_T)=p_\theta(x_T)\prod_{t=1}^Tp_\theta(x_{t-1}|x_t).
    \end{equation}
Adding conditional information, i.e., $g$, during the denoising process, $p_\theta(x_{t-1}|x_t,g)$ can be modeled as a noise prediction model, and the covariance matrix and the mean can be expressed as
    \begin{equation}\label{sigma}
        \boldsymbol\Sigma_\theta(x_t,g,t)=\iota_t\textbf{I},
    \end{equation}
    \begin{equation}\label{mu}
        \boldsymbol\mu_\theta(x_t,g,t)=\frac{1}{\sqrt{\lambda_t}}\bigg(x_t-\frac{\iota_t}{\sqrt{1-\hat{\lambda}_t}}\boldsymbol\epsilon_\theta(x_t,g,t)\bigg).
    \end{equation}
Firstly, a sample $x^T \sim \mathcal N (\textbf{0}, \textbf{I})$ is drawn from the standard normal distribution. Subsequently, sampling from the reverse diffusion chain, parameterized by $\theta$, is performed as follows
    \begin{equation}\label{sample}
        x_{t-1}|x_t=\frac{x_t}{\sqrt{\lambda_t}}-\frac{\iota_t}{\sqrt{\lambda_t(1-\hat{\lambda}_t)}}\boldsymbol\epsilon_\theta(x_t,g,t)+\sqrt{\iota_t}\boldsymbol\epsilon.
    \end{equation}
By disregarding certain weight terms, the original loss function can be streamlined and simplified to \cite{ho2020denoising}
    \begin{equation}
        \mathcal{L}_t=\Bbb{E}_{t,x_0 \sim q(x_0),\boldsymbol\epsilon\sim \mathcal N (\textbf{0}, \textbf{I})}\big[\|\boldsymbol\epsilon-\boldsymbol\epsilon_\theta(\sqrt{\hat{\lambda}_t}x_0+\sqrt{1-\hat{\lambda}_t}\boldsymbol\epsilon,t)\|^2\big].
    \end{equation}
Reverse denoising is a fundamental component that reverses the forward denoising process through the learning of a transformation kernel, denoted as $p_\theta(x_{t-1},x_t)$, which is parameterized by a deep neural network \cite{du2023deep}. This kernel facilitates restoring the original data $x_0$ by effectively removing the introduced Gaussian noise.

\renewcommand{\arraystretch}{1.3}
\begin{table}[t]
\caption{Key Mathematical Notations}
\begin{tabular}{m{0.8cm}|m{6.8cm}} 
\toprule[1.5pt]
\hline
\multicolumn{1}{c|}{\textbf{Notation}}  & \multicolumn{1}{c}{\textbf{Definition}} \\ \hline
$\theta_m, \sigma_m$ &  Type-$m$ RSUs and type-$n$ RSUs, which are based on computation and bandwidth resources, respectively \\ \hline
$\phi_{m,n}$ & Type-$(\theta_m,\sigma_n)$ RSU, which based on computation and bandwidth resources \\ \hline
$R_{m,n}$ & Reward that AVs pay for the type-$(\theta_m,\sigma_n)$ RSUs \\ \hline
$b_{m,n}$ & Bandwidth resources that the type-$(\theta_m,\sigma_n)$ RSUs provides to AVs \\ \hline
$f_{m,n}$ & CPU frequency of the type-$(\theta_m,\sigma_n)$ RSUs provides computation resources to AVs \\ \hline
$p_{m,n}$ & Transmission power between the type-$(\theta_m,\sigma_n)$ RSUs and AVs \\ \hline
$g_{m,n}$ & Channel gain between the type-$(\theta_m,\sigma_n)$ RSUs and AVs \\ \hline
$N_0$ & Noisy spectral density between the type-$(\theta_m,\sigma_n)$ RSUs and AVs \\ \hline
$T_{th}$ & Threshold for rendering capacity of AVs \\ \hline
$\zeta_1, \zeta_2$ & Wights of bandwidth and computation affect the rendering capability, respectively \\ \hline
$D, S, v$ & Resolution, spectrum efficiency, and framerate of the HMD device of AVs, respectively \\ \hline
$\mu_{m,n}$ & Effective capacitance coefficients for computational chipsets with the type-$(\theta_m,\sigma_n)$ RSUs \\ \hline
$c_{m,n}$ & Latency of unit bandwidth transmitted unit distance between the type-$(\theta_m,\sigma_n)$ RSUs and AVs \\ \hline
$d_{m,n}$ & Distance between the type-$(\theta_m,\sigma_n)$ RSUs and AVs \\ \hline
$\psi_{m,n}$ & Bandwidth cost coefficient of the type-$(\theta_m,\sigma_n)$ RSUs \\ \hline
$\xi_{m,n}$ & Unit monetary cost of the computing energy consumption of the type-$(\theta_m,\sigma_n)$ RSUs \\ \hline
$\alpha, \beta$ & User-centric parameters reflect the sensitivity of AVs to immersion and latency, respectively \\ \hline
$\delta^+, \delta^-$ & Weighting factors capture the distortion of gains and losses, respectively \\ \hline
$\tau$ & Parameter reflects the level of loss aversion \\ \hline
\bottomrule[1.5pt]
\end{tabular}\label{notation}
\end{table}

\section{Problem Formulation}\label{Problem}
In this section, we introduce a multi-dimensional contract mechanism designed to motivate RSUs to offer bandwidth and computing resources to AVs. Initially, we define the utility functions of RSUs and AVs. Subsequently, we develop a contract theory model and validate its feasibility. Finally, considering the potential for AVs irrationally in uncertain environments, we propose incorporating PT into the framework of the proposed incentive mechanism. The main mathematical notations of this paper are shown in Table \ref{notation}.

\subsection{Utility Functions}
We consider a set of RSUs denoted as $\mathcal{L}=\{1,\cdots, l,\cdots, L\}$, where $L$ denotes the total number of RSUs, accompanied by one AV, where the AV has various embodied AI twins deployed in RSUs. Given the mobility of AVs in AVs and the limited service coverage of RSUs, the embodied AI twins of AVs necessitate real-time migration across RSUs. Consequently, AVs must request resources from RSUs to which their embodied AI twins relocate, with RSUs receiving rewards for providing such resources \cite{zhong2023blockchain}. In an ideal scenario, AVs would possess specific RSU information to make more informed decisions regarding RSU rewards. However, owing to information asymmetry, AVs lack insight into the private information of RSUs \cite{wen2023task}. To mitigate this, we employ contract theory between AVs and RSUs. To begin, we outline the utilities of RSUs and AVs. Some details are shown in Fig. \ref{contract_fig}.

\subsubsection{Utilities of RSUs}
To facilitate the migration of embodied AI twins for virtual service provision, RSUs must furnish embodied AI twins of AVs with the requisite resources. In return, RSUs receive rewards from AVs, albeit at the expense of energy consumption. Consequently, the utility function of RSU $l$ is defined as the difference between the reward paid by the AV to RSU $l$ and the energy consumption cost incurred by RSU $l$, expressed as $V_l = R_l - C_l$ \cite{8239591}. As RSUs furnish both bandwidth and computing resources for embodied AI twins migration, the energy consumption cost must encompass both resource types. The cost associated with bandwidth usage, which arises from the transmission of information, is known as the bandwidth energy cost or communication cost. This cost can be expressed as $C_l^{b}=\psi_l\big(\frac{b_l}{G_l}\big)^2$, where $\psi_l$ is the bandwidth cost coefficient, $b_l$ is the bandwidth allocated by the RSU $l$ to the AV, and $G_l$ represents the channel gain between the AV and RSU $l$ \cite{5628271}. As the execution of computationally intensive tasks by the embodied AI twins within RSU necessitates the utilization of computing resources \cite{10416899}, an additional expenditure in computing costs arises. The computing energy cost incurred by RSU $l$ can be expressed as $C_l^c=\xi_l\mu_l f_l^2$, where $\xi_l$ represents the unit monetary cost of computing energy consumption, and $f_l$ denotes the CPU frequency \cite{9838422}. Consequently, based on the expression $C_l=C_l^{b}+C_l^c$, the utility of the RSU $l$ is expressed as
\begin{equation}\label{type1}
    V_l=R_l-\psi_l\bigg(\frac{b_l}{G_l}\bigg)^2-\xi_l\mu_l f_l^2.
\end{equation}

We define $\theta_l=\frac {{G_l}^2}{\psi_l}$ and $\sigma_l=\frac 1{\mu_l\xi_l}$, where $\theta_l$ and $\sigma_l$ are related to the communication cost and computation cost, respectively. Therefore, Eq. (\ref{type1}) can be varied as
\begin{equation}\label{type2}
    V_l=R_l-\frac{b_l^2}{\theta_l} -\frac{f_l^2}{\sigma_l}.
\end{equation}

According to Eq. (\ref{type2}), RSUs are classified into distinct types to delineate their heterogeneity. Specifically, RSUs can be classified as a set $\Theta=\{\theta_m,1\le m \le M\}$, representing $M$ computation cost types, and a set $\Sigma=\{\sigma_n,1\le n\le N\}$, representing $N$ communication cost types. Consequently, $MN$ RSU types exist, with their distribution described by the joint probability mass function $Q_{m,n}$, and $\sum_{m=1}^M\sum_{n=1}^NQ_{m,n}=1$ \cite{9317806}. These RSU types are arranged in non-decreasing sequences for each dimension, i.e., $\theta_1<\cdots<\theta_m<\cdots<\theta_M$ and $\sigma_1<\cdots<\sigma_n<\cdots<\sigma_N$. RSUs are differentiated based on these two cost types. For simplicity, a RSU of computation cost type $m$ and communication cost type $n$ is denoted as type-$(\theta_m,\sigma_n)$. Subsequently, we omit the subscript $i$ and utilize the combination of bandwidth, CPU frequency and reward, i.e., $\{b_{m,n}, f_{m,n}, R_{m,n}\}$, to express the utility of the type-$(\theta_m,\sigma_n)$ RSUs as
\begin{equation}\label{Vmn}
    V_{m,n}=R_{m,n}-\frac{b_{m,n}^2}{\theta_m}-\frac{f_{m,n}^2}{\sigma_n}.
\end{equation}

\subsubsection{Utility of AVs}
After embodied AI twins are migrated to RSUs their service scope covers AVs, and the embodied AI twins will obtain resources to perform tasks, allowing AVs to obtain in-vehicle services \cite{zhong2023blockchain,10416899}. In \cite{10144339}, the authors introduced a new metric called ``Meta-Immersion" to measure the Quality of Experience (QoE) experienced by AVs in virtual services. In our paper, we also employ this virtual immersive metric to measure the satisfaction of AVs receiving in-vehicle services from RSUs. In addition, it takes time for RSUs to transmit service data to AVs, which may cause latency, resulting in a degradation of the service experience of AVs. Therefore, the utility function of AVs should consider the immersion metric, the latency, and the reward, i.e., 
\begin{equation}
    U=\sum_{m=1}^M\sum_{n=1}^N(\alpha M_{m,n}-\beta D_{m,n}-R_{m,n}),
\end{equation}
where $\alpha$ and $\beta$ are user-centric parameters that can reflect the sensitivity of the users within the AV to the immersion indicator and latency, respectively. $M_{m,n}$ represents the immersion metric of the users within the AV achieved from the type-$(\theta_m,\sigma_n)$ RSUs, and $D_{m,n}$ denotes the latency of the AV receiving in-vehicle services from the type-$(\theta_m,\sigma_n)$ RSUs. A viable mathematical expression for immersion metric can be derived by taking the connectivity coefficient multiplied by the logarithm of the stimulus intensity \cite{9999298}. We consider the downlink data rate and rendering capacity as the connectivity coefficient and the stimulus intensity \cite{9999298,9838736}, respectively. The downlink data rate from the type-$(\theta_m,\sigma_n)$ RSUs to the AV is expressed as
\begin{equation}
    r_{m,n}=b_{m,n}\ln\bigg(1+\frac{p_{m,n}|g_{m,n}|^2}{b_{m,n}N_0}\bigg),
\end{equation}
where $p_{m,n}$, $g_{m,n}$, and $N_0$ represent the transmission power, the channel gain, and the noisy spectral density between the type-$(\theta_m,\sigma_n)$ RSUs and AVs, respectively. AVs obtain immersive experiences in the physical world through Head-Mounted Display (HMD). The HMD device of users in AVs determines the rendering capability (in units of resolution $D$ and frame rate $v$) of the provided virtual service for users within the AV, which can be expressed as \cite{9838736}
\begin{equation}
    t_{m,n}=\ln\bigg(\frac{Dv\big(\zeta_1Sb_{m,n}+\zeta_2\mu_{m,n}f_{m,n}^2\big)}{T_{th}}\bigg),
\end{equation}
where $S$ represents the spectrum efficiency of the HMD device of AVs, while $\mu_{m,n}$ signifies the effective capacitance coefficient for the computing chipset associated with the type-$(\theta_m,\sigma_n)$ RSUs. The weights $\zeta_1$ and $\zeta_2$ are greater than zero and $\zeta_1+\zeta_2=1$, ensuring proper weighting. Building upon this analysis, we define the immersion metric as
\begin{equation}\label{M}
\begin{split}
    M_{m,n}=&b_{m,n}\ln\bigg(1+\frac{p_{m,n}|g_{m,n}|^2}{b_{m,n}N_0}\bigg) \\
    &\ln\bigg(\frac{Dv\big(\zeta_1Sb_{m,n}+\zeta_2\mu_{m,n}f_{m,n}^2\big)}{T_{th}}\bigg).
\end{split}
\end{equation}

The transmission latency between RSUs and AVs arises due to factors such as distance and the available bandwidth for data transmission. Therefore, the latency can be quantified by considering the distance between RSUs and AVs, as well as the bandwidth capacity required for transmitting data \cite{10301793}
\begin{equation}\label{Dmn}
    D_{m,n}=c_{m,n}d_{m,n}b_{m,n},
\end{equation}
where $c_{m,n}$ and $d_{m,n}$ represent the latency of unit bandwidth transmitted unit distance and the distance between the type-$(\theta_m,\sigma_n)$ RSUs and the AV, respectively. Based on Eqs. (\ref{U}), (\ref{M}), and (\ref{Dmn}), the utility of the AV receives resources from the type-$(\theta_m,\sigma_n)$ RSUs for embodied AI twins migration is expressed as
\begin{equation}
\begin{split}
    U_{m,n}=&\alpha b_{m,n}\ln\bigg(1+\frac{p_{m,n}|g_{m,n}|^2}{b_{m,n}N_0}\bigg) \\
    &\ln\bigg(\frac{Dv\big(\zeta_1Sb_{m,n}+\zeta_2\mu_{m,n}f_{m,n}^2\big)}{T_{th}}\bigg) \\
    &-\beta c_{m,n}d_{m,n}b_{m,n}-R_{m,n}.
\end{split}
\end{equation}
The expected utility of the AV for all types of RSUs based on EUT is expressed as
\begin{equation}
    U^{EUT}=\sum_{m=1}^M\sum_{n=1}^NQ_{m,n}U_{m,n}.
\end{equation}

Since the interaction between AVs and RSUs may be uncertain, relying solely on EUT for evaluating the utilities of AVs may not be practical. Instead of strictly adhering to rational decision-making, which aims to maximize utility under EUT, AVs often make decisions on relative rather than absolute utility, a principle known as PT. According to PT, AVs prioritize relative gains and losses in uncertain environments, where outcomes are not easily predictable, over maximizing utility. We introduce PT-based utility function construction for the AV to the type-($\theta_m,\sigma_n$) RSUs as 
\begin{equation}
    U_{m,n}^{PT}= \left\{ \begin{aligned}
    (U_{m,n}^{EUT}-U_{m,n}^{ref})^{\delta^{+}},\: U_{m,n}^{EUT}\geq U_{m,n}^{ref},\\
    - \kappa (U_{m,n}^{ref}-U_{m,n}^{EUT})^{\delta^{-}},\: U_{m,n}^{EUT}<U_{m,n}^{ref},\\
\end{aligned} \right.
\end{equation}
where $0\le \delta^{+}\le 1$ and $0\le \delta^{-}\le 1$ are the weighting parameters of gain and loss distortions, respectively \cite{huang2021efficient}. $\kappa\ge 0$ denotes the aversion parameter and $U_{m,n}^{ref}$ represents the reference utility of the AV for the type-($\theta_m,\sigma_n$) RSUs. If the EUT-based utility of the AV for the type-($\theta_m,\sigma_n$) RSUs exceeds $U_{m,n}^{ref}$, a gain is obtained. Otherwise, a loss is obtained. The PT-based expected utility of the AV for all RSUs is expressed as
\begin{equation}\label{U}
    U^{PT}=\sum_{m=1}^M\sum_{n=1}^NQ_{m,n}U_{m,n}^{PT}.
\end{equation}
Computing the utility based on PT allows AVs to adapt more effectively to dynamic conditions, enhancing their decision-making capabilities.

\subsection{Contract Formulation and Feasibility}
\begin{figure}[t]
\centering
\includegraphics[width=0.45\textwidth]{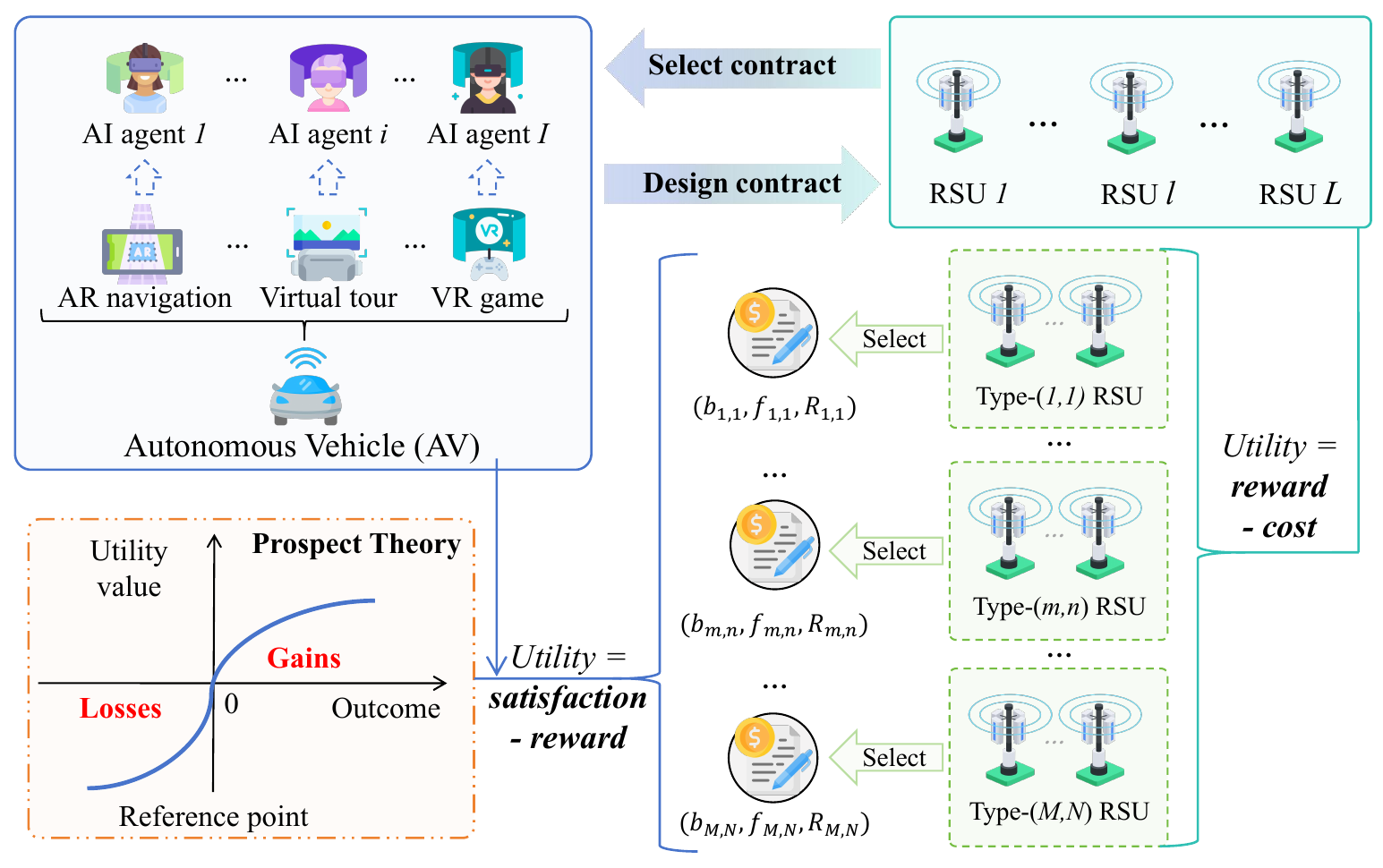}
\captionsetup{font=footnotesize}
\caption{An illustration for the utilities based on prospect theory in the contract theory model.}
\label{contract_fig}
\end{figure}

We formulate a contract model and find the optimal contract for single-AV and multi-RSU scenarios.

\subsubsection{Contract Formulation}
AVs determine the reward payment for RSUs based on their computing and bandwidth capabilities. However, since RSUs protect their cost information, AVs may lack awareness of the two types of costs associated with each RSU. To address this information asymmetry, we propose a contract theory model between AVs and RSUs as an incentive mechanism to encourage RSUs to offer resources. The AV is the principal designing the contract, while RSUs are the agents selecting the optimal contracts. The contract item is denoted as $\boldsymbol\rho=\{b_{m,n},f_{m,n},R_{m,n},1\le n\le N,1\le m\le M\}$, and the type-$(\theta_m,\sigma_n)$ RSUs should choose the item $(b_{m,n},f_{m,n},R_{m,n})$ tailed for it \cite{zhang2020contracts}. Considering Eq. (\ref{Vmn}), we use the symbol $V_{m,n}^{p,q}$ to denote the type-$(\theta_m,\sigma_n)$ RSUs selects the contract item $\{b_{p,q},f_{p,q},R_{p,q}\}$, which is design for the type-$(\theta_p,\sigma_q)$ RSUs. $V_{m,n}^{p,q}$ is expressed as
\begin{equation}
    V_{m,n}^{p,q}=R_{p,q}-\frac{b_{p,q}^2}{\theta_m}-\frac{f_{p,q}^2}{\sigma_n}.
\end{equation}
To ensure the optimal selection of contract terms suitable for each RSU type, it is essential to satisfy Individual Rationality (IR) and Incentive Compatibility (IC) constraints, which are defined as follows \cite{wen2023task}.
\begin{definition}
{\textbf{(Individual Rationality):}} All the type-$(\theta_m,\sigma_n)$ RSUs guarantee a non-negative utility of their selected contract item $\{b_{m,n},f_{m,n},R_{m,n},1\le n\le N,1\le m\le M\}$, i.e.,
\begin{equation}\label{IR}
\begin{split}
    V_{m,n}^{m,n}=R_{m,n}-\frac{b_{m,n}^2}{\theta_m}-\frac{f_{m,n}^2}{\sigma_n}\ge 0,\\
     \: 1\le m\le M,\: 1\le n\le N.
\end{split}
\end{equation}
\end{definition}

\begin{definition}
{\textbf{(Incentive Compatibility):}} All the type-$(\theta_m,\sigma_n)$ RSUs prefer to select the contract item $\{b_{m,n},f_{m,n},R_{m,n},1\le n\le N,1\le m\le M\}$ designed for its type rather than any other contract item $\{b_{i,j},f_{i,j},R_{i,j},1\le i\le N,1\le j\le M\}$, where $m\neq i$ and $n\neq j$, i.e.,
\begin{equation}\label{IC}
\begin{split}
    V_{m,n}^{m,n}\ge\max\{V_{m,n}^{i,n},V_{m,n}^{m,j},V_{m,n}^{i,j}\},\:  \: 1\le m,i\le M,\\
    \text{and}\: m\neq i,\:\text{and}\:   \:1\le n, j\le N,\:\text{and}\: n\neq j.    
\end{split}
\end{equation}
\end{definition}

By imposing IR constraints, the system ensures the active participation of RSUs, while the IC constraints guarantee that each RSU selects the contract item specifically designed for its type, aiming to maximize the benefits obtained. Through the integration of both IR and IC constraints, the AV seeks to enhance their expected utility \cite{wen2023task}. The problem of maximizing the expected utility of the AV is formulated as
\begin{equation}\label{problem1}
    \begin{split}
    \textbf{Problem:}\:&\max\limits_{\boldsymbol{b}_{m,n},\boldsymbol{f}_{m,n},\boldsymbol{R}_{m,n}}\:U^{PT}  \\
    &\:\:\text{s.t.}\:\: (\ref{IR})\: \text{and} \:(\ref{IC}),\\
    &\:\:\:\:\:\:\:\:\: b_{m,n}\ge 0,f_{m,n}\ge 0,R_{m,n}\ge 0,\\
    &\:\:\:\:\:\:\:\:\:\theta_m\ge 0,\sigma_n\ge 0,
    \end{split}
\end{equation}
where $\boldsymbol{b}_{m,n}=\{b_{1,1},\cdots,b_{m,n},\cdots,b_{M,N}\}$, $\boldsymbol{f}_{m,n}=\{f_{1,1},\cdots,f_{m,n},\cdots,f_{M,N}\}$, and $\boldsymbol{R}_{m,n}=\{R_{1,1},\cdots,R_{m,n},\cdots,R_{M,N}\}$, which are the design of the contracts for all type of RSUs.

\subsubsection{Contract Feasibility}
Considering (\ref{problem1}), it is evident that the problem formulated is a multi-dimensional non-convex optimization problem \cite{zhang2020contracts}. With $MN$ IR constraints and $MN(MN-1)$ IC constraints, solving this problem directly becomes challenging. Consequently, constraint reduction becomes imperative. We study the properties of $V_{m,n}^{m,n}$ and derive the following Lemmas to validate the feasibility of the proposed contract.

\begin{lemma}\label{lemma1}
    For $ \: 1\le m,i\le M$ and $ \: 1\le n,j\le N$, if $m>i$ and $n>j$, we have 
    \begin{equation}
    \begin{split}
        b_{i,j}&\le\max\{b_{i,n},b_{m,j}\}\le b_{m,n}, \\
        f_{i,j}&\le\max\{f_{i,n},f_{m,j}\}\le f_{m,n}.
    \end{split}
    \end{equation}
\end{lemma}
\begin{proof}
    Please refer to \cite{zhang2020contracts}.
\end{proof}

\begin{lemma}\label{lemma2}
    For $ \: 1\le m\le M$ and $ \: 1\le n\le N$, there are      $V_{m,n}^{m,n}\ge V_{m,n-1}^{m,n-1},\:V_{m,n}^{m,n}\ge V_{m-1,n}^{m-1,n}$, and $\:V_{m,n}^{m,n}\ge V_{m-1,n-1}^{m-1,n-1}$.
\end{lemma}
\begin{proof}
    Please refer to \cite{zhang2020contracts}.
\end{proof}

\begin{lemma}\label{lemma3}
    For the type-${(\theta_{m+1},\sigma_{n+1})}$ RSU, there are $V^{m,n}_{m+1,n+1}\ge V^{m,n-1}_{m+1,n+1},\:V^{m,n}_{m+1,n+1}\ge V^{m-1,n}_{m+1,n+1}$ and $\:V^{m,n}_{m+1,n+1}\ge V^{m-1,n-1}_{m+1,n+1}$, for $ \: 1\le n\le N$ and $ \: 1\le m\le M$.
\end{lemma}
\begin{proof}
    Please refer to \cite{zhang2020contracts}.
\end{proof}

\begin{lemma}\label{lemma4}
    The IR constraint defined in (\ref{IR}) can be reduced as $V_{1,1}^{1,1}>0$.
\end{lemma}
\begin{proof}
    Please refer to \cite{zhang2020contracts}.
\end{proof}

\begin{lemma}\label{lemma5}
    For the type-${(\theta_m,\sigma_n)}$ RSUs, the IC constraints can be reduced as Local Downward Incentive Compatibility (LDIC), shown as
    \begin{equation}
        \begin{split}
            V_{m,n}^{m,n}\ge \max\{V_{m,n}^{m,n-1},V_{m,n}^{m-1,n},V_{m,n}^{m-1,n-1}\},\\
        2\le m\le M, \: 2\le n\le N,
        \end{split}
    \end{equation}
    and Local Upward Incentive Compatibility (LUIC), shown as
    \begin{equation}
        \begin{split}
            V_{m,n}^{m,n}\ge \max\{V_{m,n}^{m,n+1},V_{m,n}^{m+1,n},V_{m,n}^{m+1,n+1}\},\\
            1\le m\le M-1, \: 1\le n\le N-1.
        \end{split}
    \end{equation}
\end{lemma}
\begin{proof}
    Please see the proof in Appendix A.
\end{proof}

\begin{remark}
Lemma \ref{lemma1} suggests that AVs will request more resources from RSUs with higher types (i.e., lower costs), and vice versa. According to Lemma \ref{lemma2} and Lemma \ref{lemma3}, it can be inferred that RSUs with higher types (i.e., lower costs), have the potential for higher profits. Additionally, Lemma \ref{lemma4} establishes that if the lowest type RSU meets the IR constraint, then IR constraints of all RSUs will also hold. Lemma \ref{lemma5} means that if the IC constraints hold between type-$(\theta_m, \sigma_n)$ RSU and the next lower (higher) type RSU, then they also hold between type-$(\theta_m, \sigma_n)$ RSU and other RSUs of lower (higher) type. In other words, Lemma \ref{lemma5} asserts that IC constraints between RSU types cascade downwards (upwards), simplifying the fulfillment of constraints. Therefore, Lemmas \ref{lemma1}, \ref{lemma4}, and \ref{lemma5} serve as both necessary and sufficient conditions for IR and IC constraints, effectively reducing the total constraints.
\end{remark}

Based on the above analysis of the Lemmas, we can derive the utility of the type-($\theta_m,\sigma_n$) RSUs, which is shown as follows:

\begin{theorem}\label{theorem1}
    For $ \: 1\le m,i\le M$ and $ \: 1\le n,j\le N$, when $m>i$ and $n>j$, the utility of the type-($\theta_m,\sigma_n$) RSUs can be expressed as
    \begin{equation}\label{reduce1}
    \begin{split}
        V_{m,n}^{m,n}=&\sum_{i=1}^{m-1}\sum_{j=1}^{n-1}\big(\Delta_ib_{i,j}^2+\Lambda_jf_{i,j}^2\big)+\sum_{i=1}^{m-1}\sum_{j=1}^{n-1}\max\\
        &\bigg\{0,\Delta_i\big(b_{i,j+1}^2-b_{i,j}^2\big), \Lambda_j\big(f_{i+1,j}^2-f_{i,j}^2\big)\bigg\},
    \end{split}
    \end{equation}
    where $\Delta_i=\frac{1}{\theta_i}-\frac{1}{\theta_{i+1}}>0$, and $\Lambda_j=\frac{1}{\sigma_j}-\frac{1}{\sigma_{j+1}}>0$.
\end{theorem}
\begin{proof}
    Please see the proof in Appendix B.
\end{proof}

Furthermore, we can find the optimal contract $(b_{m,n}^*, f_{m,n}^*, R_{m,n}^*),   1\le m\le M, 1\le n\le N$, which shown as follows.
\begin{theorem}
    The optimal contract $(b_{m,n}^*, f_{m,n}^*, R_{m,n}^*),   1\le m\le M, 1\le n\le N$ for the type-($\theta_m,\sigma_n$) RSUs is expressed as
    \begin{equation}
    \begin{split}
        (b_{m,n}^*, f_{m,n}^*)=&\arg\max_{\small{(\boldsymbol b_{m,n}, \boldsymbol f_{m,n})}}\sum_{m=1}^M\sum_{n=1}^N\bigg(\alpha M_{m,n}-\\
        &\beta D_{m,n}-\bigg(V_{m,n}+\frac{b_{m,n}^2}{\theta_m}-\frac{f_{m,n}^2}{\sigma_n}\bigg)\bigg).
    \end{split}
    \end{equation}
    From Eq. (\ref{reduce1}), we can obtain the $V_{m,n}^*$ when $b_{m,n}=b_{m,n}^*$ and $f_{m,n}=f_{m,n}^*$. Consequently, we can get the optimal reward $R_{m,n}^*$, which is shown as
    \begin{equation}
        R_{m,n}^*=V_{m,n}^*+\frac{b_{m,n}^{*2}}{\theta_m}-\frac{f_{m,n}^{*2}}{\sigma_n}.
    \end{equation}
\end{theorem}

\section{Generative Diffusion Model for Optimal Contract Design}\label{Optimial_Contract}

\begin{figure*}[t]
\centering
\includegraphics[width=0.9\textwidth]{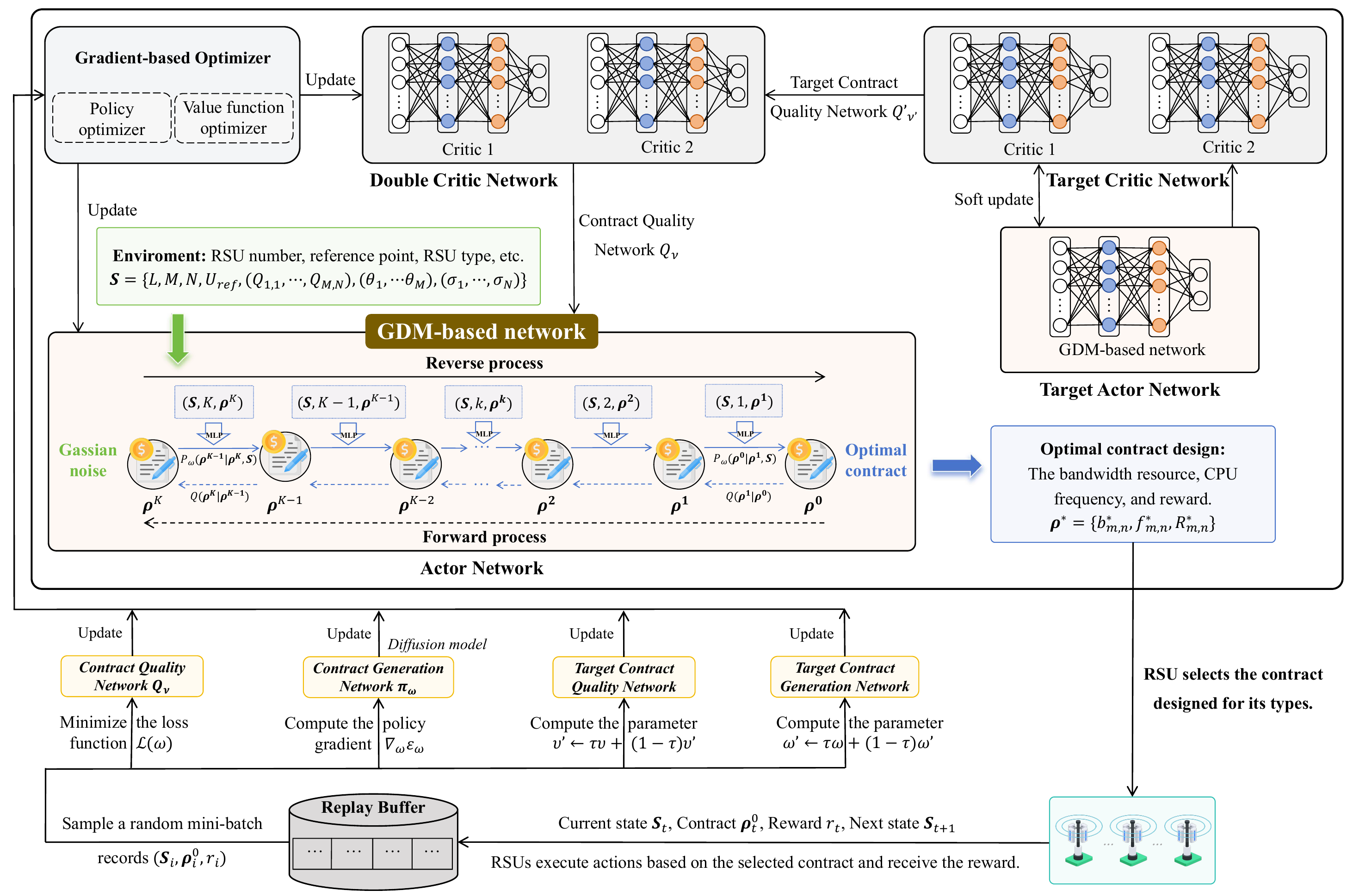}
\captionsetup{font=footnotesize}
\caption{GDM-based framework to find the optimal contract designs.}
\label{GDM}
\end{figure*}

GDMs seek to enhance contract design by iteratively refining initial distributions through denoising. Thus, we endorse employing a GDM-based approach to ascertain the optimal contract item, i.e., $\boldsymbol\rho^*=\{b_{m,n}^*,f_{m,n}^*,R_{m,n}^*\}$. Specifically, GDMs encompass both forward and reverse diffusion processes. In the forward diffusion, Gaussian noise incrementally augments an initial sample, i.e., $\boldsymbol\rho^0=\{b_{m,n}^0,f_{m,n}^0,R_{m,n}^0\}$. This process unfolds iteratively, typically represented as a Markov process with $K$ steps, yielding a sequence of samples $\{\boldsymbol\rho^1,\cdots,\boldsymbol\rho^k,\cdots,\boldsymbol\rho^K\}$ \cite{du2023deep}. The forward diffusion process of GDMs can be described as
\begin{equation}\label{forward}
    \begin{split}
        Q(\boldsymbol\rho^1,&\boldsymbol\rho^2,\cdots,\boldsymbol\rho^K|\boldsymbol\rho^0)=\prod_{k=1}^K Q(\boldsymbol\rho^k|\boldsymbol\rho^{k-1}),\\
        Q(\boldsymbol\rho^k|\boldsymbol\rho^{k-1})&=\mathcal{N}(\boldsymbol\rho^k;\boldsymbol\mu_k=\sqrt{1-\iota_k}\boldsymbol\rho^{k-1},\boldsymbol\Sigma_k=\iota_k\textbf{I}),
    \end{split}
\end{equation}
where $\iota_k\in(0,1)$ is a pre-determined hyperparameter. From Eq. (\ref{forward}), we can obtain that the sample $\boldsymbol\rho^k$ at the $k$-th step obeys Gaussian distribution, with a mean of $\boldsymbol\mu_k$ and a variance of $\boldsymbol\Sigma_k$. When $\iota_k$ is sufficiently small, the probability distribution of reverse diffusion process $Q(\boldsymbol\rho^{k-1}|\boldsymbol\rho^k,\boldsymbol\rho^0)$ is consistent with the posterior probability distribution of the forward diffusion process. The parametric model $P_\omega$ can be used to approximate $Q(\boldsymbol\rho^{k-1}|\boldsymbol\rho^k,\boldsymbol\rho^0)$ as
\begin{equation}
\begin{split}
    P_\omega(\boldsymbol\rho^0,\boldsymbol\rho^1,\cdots,\boldsymbol\rho^K)=P_\omega(\boldsymbol\rho^K)\prod_{k=1}^KP_\omega(\boldsymbol\rho^{k-1}|\boldsymbol\rho^k),\\
    P_\omega(\boldsymbol\rho^{k-1}|\boldsymbol\rho^k)=\mathcal{N}(\boldsymbol\rho^{k-1};\boldsymbol\mu_\omega(\boldsymbol\rho^k,k),\boldsymbol\Sigma_\omega(\boldsymbol\rho^k,k)),
\end{split}
\end{equation}
where $\omega$ represents the model parameters and $P(\omega_K)=\mathcal{N}(\omega_K;\textbf{0},\textbf{I})$. Therefore, the reverse diffusion process can be accomplished by estimating $Q(\boldsymbol\rho^{k-1}|\boldsymbol\rho^k,\boldsymbol\rho^0)$ using highly trained $P_\omega(\boldsymbol\rho^{k-1}|\boldsymbol\rho^k)$. In contract modeling, the environment encompasses various factors that influence the optimal contract design, which is defined as 
\begin{equation}
\begin{split}
    \boldsymbol{S}:=&\big\{L, \:M,\: N,\: U_{ref},\: (Q_{1,1},\cdots,Q_{M,N}), \\
    &(\theta_1,\cdots,\theta_m,\cdots,\theta_M), \:(\sigma_1,\cdots,\sigma_n,\cdots,\sigma_N)\big\}. 
\end{split}
\end{equation}
The diffusion model network is denoted as $\pi_\omega(\boldsymbol\rho|\boldsymbol{S})$, using weights $\omega$ to map environmental states to contract designs as the contract design policy. The objective of $\pi_\omega(\boldsymbol\rho|\boldsymbol{S})$ is to generate a deterministic contract design aimed at optimizing the expected cumulative reward across multiple time steps. This optimization process aligns with the representation of the reverse diffusion process, which is expressed as \cite{du2023ai}
\begin{equation}
    \begin{split}
        \pi_\omega(\boldsymbol\rho|\boldsymbol{S})=&P_\omega(\boldsymbol\rho^0,\boldsymbol\rho^1,\cdots,\boldsymbol\rho^K|\boldsymbol{S})\\
        =&\mathcal{N}(\boldsymbol\rho^K;\textbf{0},\textbf{I})\prod_{k=1}^KP_\omega(\boldsymbol\rho^{k-1}|\boldsymbol\rho^k,\boldsymbol{S}),
    \end{split}
\end{equation}
where
\begin{equation}
\begin{split}
    P_\omega(\boldsymbol\rho^{k-1}|\boldsymbol\rho^k,\boldsymbol{S})&=\mathcal{N}\big(\boldsymbol\rho^k;\boldsymbol\mu_\omega(\boldsymbol\rho^k,\boldsymbol{S},k),\boldsymbol\Sigma_\omega(\boldsymbol\rho^k,\boldsymbol{S},k)\big),\\
    \boldsymbol\mu_\omega(\boldsymbol\rho^k,\boldsymbol{S},k)&=\frac{1}{\sqrt{\lambda_k}}\bigg(\boldsymbol\rho^k-\frac{\iota_t}{\sqrt{1-\hat{\lambda}_k}}\boldsymbol\varepsilon_\omega(\boldsymbol\rho^k,\boldsymbol{S},k)\bigg),\\
    \boldsymbol\Sigma_\omega(\boldsymbol\rho^k,\boldsymbol{S},k)&=\iota_t\textbf{I},
\end{split}
\end{equation}
where $\boldsymbol\varepsilon_\omega$ denotes the contract generation network, $\lambda_k:=1-\iota_k$ and $\hat{\lambda}_k:=\prod_{i=0}^k\iota_i$. Analyzing similar to Eq. (\ref{sample}), we can obtain
\begin{equation}\label{contract}
    \boldsymbol\rho^{k-1}|\boldsymbol\rho^k=\frac{\boldsymbol\rho^k}{\sqrt{\lambda_k}}-\frac{\iota_k}{\sqrt{\lambda_k(1-\hat{\lambda}_k)}}\boldsymbol\varepsilon_\omega(\boldsymbol\rho^k,\boldsymbol{S},k)+\sqrt{\iota_k}\boldsymbol\varepsilon.
\end{equation}
Subsequently, the contract quality network $Q_\upsilon$ is introduced, serving to map an environment-contract pair $(\boldsymbol{S},\boldsymbol\rho)$. Here, $Q_\upsilon$ denotes the anticipated cumulative reward when an agent selects a contract design policy from the current state and subsequently adheres to it. Consequently, the optimal contract design policy can be attained by maximizing entropy \cite{wen2024diffusion}
\begin{equation}\label{loss}
    \pi=\arg\max_{\pi_\omega}\sum_{t=0}^{T}\Bbb E\big[\gamma^t\big(r(\boldsymbol{S}_t,\boldsymbol\rho_t)+\varpi E(\pi_\omega(\boldsymbol{S}_t))\big)\big],
\end{equation}
where $\gamma$ denotes the discount factor, and $\varpi$ represents the hyperparameter of the temperature coefficient, which is used to adjust the emphasis on entropy. $E(\pi_\omega(\boldsymbol{S}_t))$ is the entropy value of $\pi_\omega(\boldsymbol{S}_t)$, which is expressed as
\begin{equation}
    E(\pi_\omega(\boldsymbol{S}_t))=-\pi_\omega(\boldsymbol{S}_t)\log\pi_\omega(\boldsymbol{S}_t).
\end{equation}
$r(\boldsymbol{S}_t,\boldsymbol\rho_t)$ represents the immediate reward when executing action $\boldsymbol\rho_t$ in state $\boldsymbol{S}_t$. Based on the IR and IC constraints, the reward is designed as
\begin{equation}
\begin{split}
    r(\boldsymbol{S}_t,\boldsymbol\rho_t)&=U^{PT}+\sum_{m=1}^M\sum_{n=1}^NV_{m,n}^{m,n} \\
    &+\sum_{m=1}^M\sum_{n=1}^N\sum_{i=1}^M\sum_{j=1}^N(V_{m,n}^{m,n}-V_{m,n}^{i,j}), i\neq m, j\neq n.
\end{split}
\end{equation}

Moreover, the contract quality network $Q_\upsilon$ is trained conventionally by minimizing the Bellman operator using the double Q-learning technique \cite{hasselt2010double}. This involves constructing $Q_{v_1}$ and $Q_{v_2}$ networks, along with their corresponding target networks $Q_{v_1 '}$, $Q_{v_2 '}$, and $\pi_{\omega '}$. The optimization of $v_j, j=\{1,2\}$ is achieved by minimizing the objective function \cite{wen2024diffusion}
\begin{equation}\label{object}
\begin{split}
    \Bbb E_{(\boldsymbol{S}_t,\boldsymbol\rho_{t},\boldsymbol{S}_{t+1},r_t)\sim\mathcal{H}_t}\bigg [\sum_{j=1,2}(r(\boldsymbol{S}_t,\boldsymbol\rho_t)-Q_{v_j}(\boldsymbol{S}_t,\boldsymbol\rho_t)\\
    +\gamma^t(1-d_{t+1})\pi_{\omega'}(\boldsymbol{S}_{t+1})Q_{v '_j}'(\boldsymbol{S}_{t+1})\big)^2\bigg],
\end{split}
\end{equation}
where $Q_{v '_j}'(\boldsymbol{S}_{t+1})=\min\{Q_{v_1 '}(\boldsymbol{S}_{t+1}),Q_{v_2 '}(\boldsymbol{S}_{t+1})\}$. $\mathcal{H}_t$ is a randomly sampled mini-batch of transitions retrieved from the replay buffer $\mathcal{D}$ during
neural network training step $t$, and $d_{t+1}$ represents the termination flag, which is a $0-1$ variable. 

The contract design algorithm utilizes denoising technology to produce the optimal contract design. Subsequently, exploration noise is introduced to the contract design and implemented to accumulate exploration experience \cite{du2023ai}. The detail of the GDM-based optimal contract design algorithm is shown in \textbf{Algorithm \ref{Contract_design}} and the illustration is shown in Fig. \ref{GDM}.

\begin{algorithm}
\caption{GDM-based Optimal Contract Design}\label{Contract_design}
\begin{algorithmic}[1] 
\State \textit{\textbf{Phase 1 - Training}}
\State Number of iterations $K$ to add noise, mini-batch size $H$, discount factor $\gamma$, exploration noise $\boldsymbol{\epsilon}$, soft target update parameter $\tau$.
\State \textit{$\#$ Initialize Parameters}
\State Initialize contract generation network $\boldsymbol{\varepsilon}_\omega$, contract quality network $Q_\upsilon$, target contract generation network $\boldsymbol{\varepsilon}'_{\omega'}$, and target contract quality network $Q'_{v'}$ with weights $\omega$, $v$, $\omega'$, and $v'$, respectively. \\
Initialize replay buffer $\mathcal{D}$.
\State \textit{$\#$ Learn contract design network}
\For {\rm{episode} $=1$ \textbf{to} max episode $Z$}
\For {\rm{step} $t=1$ \textbf{to} max step $T$}
\State Input the current environment $\boldsymbol{S}_t$.
\State Set $\boldsymbol\rho^K_t$ as Gaussian noise.
\State Generate contract design $\boldsymbol\rho^0_t$ by denoising $\boldsymbol\rho^K_t$ based on Eq. (\ref{contract}).
\State Add the exploration noise $\boldsymbol{\epsilon}$ to $\boldsymbol\rho^0_t$.
\State Compute reward $r_t$, i.e., the utility of user based on Eq. (\ref{U}), by executing contract design $\boldsymbol\rho^0_t$.
\State Store the record ($\boldsymbol{S}_t,\boldsymbol\rho^0_t,r_t, \boldsymbol{S}_{t+1}$) in replay buffer $\mathcal{D}$.
\State Sample a random mini-batch $\mathcal{H}$ of $H$ records ($\boldsymbol{S}_i,\boldsymbol\rho^0_i,r_i$) from replay buffer $\mathcal{D}$.
\State Update the contract quality network $Q_\upsilon$ by minimizing the object function based on Eq. (\ref{object}).
\State Update the contract generation network $\boldsymbol{\varepsilon}_\omega$ by computing the policy gradient based on Eq. (\ref{loss}).
\State Update target contract quality network $Q'_{v'}$ by updating the parameter $v'\leftarrow\tau v+(1-\tau)v'$.
\State Update target contract generation network $\boldsymbol{\varepsilon}'_{\omega'}$ by updating the parameter $\omega'\leftarrow\tau\omega+(1-\tau)\omega'$.
\EndFor
\EndFor
\State \textit{\textbf{Phase 2 - Inference}}
\State Input the environment $\boldsymbol{S}$.
\State \textit{$\#$ Generate optimal contract items}
\State Generate contract design $\boldsymbol\rho^0$ by denoising $\boldsymbol\rho^K$ based on Eq. (\ref{contract}).\\
\Return The optimal contract design $\boldsymbol\rho^0$.
\end{algorithmic}
\end{algorithm}

\textbf{Algorithm \ref{Contract_design}} employs a contract design algorithm based on GDMs, which iteratively learns and gathers experience from the environment through exploration \cite{du2023ai}. During the training phase, the contract design network undergoes iterative training processes to enhance its capabilities. Then, during the inference phase, the trained network applies its learned knowledge to generate optimal contract items based on the current environment. Leveraging the characteristics of the diffusion model, the algorithm dynamically adapts its output while seeking the optimal solution. This iterative process enhances the algorithm's robustness and efficiency compared to neural network models that provide direct output solutions.
We denote the weight counts of the contract generation network and the contract quality network as $\varkappa$ and $\varrho$, respectively. The initialization complexity is $\mathcal{O}(2\varkappa + 2\varrho)$. The complexity for action generation is $\mathcal{O}(K\varkappa)$ per step, considering $K$ denoising steps. The storage complexity of the replay buffer operation is $\mathcal{O}(1)$ and the complexity of mini-batch sampling is $\mathcal{O}(H)$. Updating the target contract generation network and the target contract quality network incurs complexities of $\mathcal{O}(\varkappa)$ and $\mathcal{O}(\varrho)$ per update \cite{du2023user}, respectively. Consequently, the computational complexity of the training phase is $\mathcal{O}(ZT(K\varkappa + \varrho))$, and for the inference phase, it is $\mathcal{O}(\varkappa)$ \cite{du2023user}.

\section{Numerical Results}\label{Results}
In this section, we provide numerical results to empirically demonstrate the effectiveness of the proposed approach. Similar to \cite{10416899, 9838736}, the key parameters of the experiment are delineated in Table \ref{exp}. We assume that there are $5$ RSUs in total, categorized into $2$ types based on computation resources and $2$ types based on bandwidth resources, resulting in $4$ distinct types of RSUs, i.e., $L=5,\:M=2,\:N=2$. $\theta_1$ and $\theta_2$ are randomly sampled in the range of $[10, 100]$ and $[100, 200]$, respectively. Similarly, $10\le \sigma_1\le 100\le \sigma_2\le 200$.

\begin{table}[ht]\label{parameter}
	\renewcommand{\arraystretch}{1.2}
        \captionsetup{font = small}
	\caption{ Key Parameters in the Simulation. }\label{table} \centering 
	\begin{tabular}{m{5.7cm}<{\raggedright}|m{1.7cm}<{\centering}}	 	
		\hline	
            \textbf{Parameters} & \textbf{Values}\\	
		\hline
            Resolution of HMD devices of users in AVs $D$ & $2160 \times 1200$ \\
            \hline
            Framerate of HMD devices of users in AVs $v$ & $90$ \\
            \hline
            Spectrum efficiency of HMD devices of users in AVs $S$ & $[1,3]$ \\
            \hline
            Transmission power between the type-$(\theta_m,\sigma_n)$ RSUs and AVs $p_{m,n}$ & $[20, 25]\rm{dBm}$ \\ 
            \hline
            Channel gain between the type-$(\theta_m,\sigma_n)$ RSUs and AVs $g_{m,n}$ & $[-25,-22]\rm{dB}$ \\ 
            \hline
            Noisy spectral density between the type-$(\theta_m,\sigma_n)$ RSUs and AVs $N_0$ & $-95\rm{dBm}$\\
            \hline
		Learning rate of the contract generation network &  $2\times10^{-7}$\\	
		\hline
		Learning rate of the contract quality network  &  $2\times10^{-7}$  \\	
		\hline
            Maximum capacity of the replay buffer $\mathcal{D}$ &  $ 10^6$\\
		\hline		
            Number of iterations to add noise $K$ &  $3$\\
		\hline
  		Mini-batch size $H$ &  $512$\\	
		\hline	
  		Discount factor $\gamma$&  $1$\\
		\hline
  		Exploration noise $\boldsymbol\epsilon$  &  $0.01$  \\	
		\hline		
  		Soft target update parameter $\tau$   & $0.005$  \\
		\hline
	\end{tabular}\label{exp}
\end{table}



\begin{figure}[t]
\centering
\includegraphics[width=0.5\textwidth]{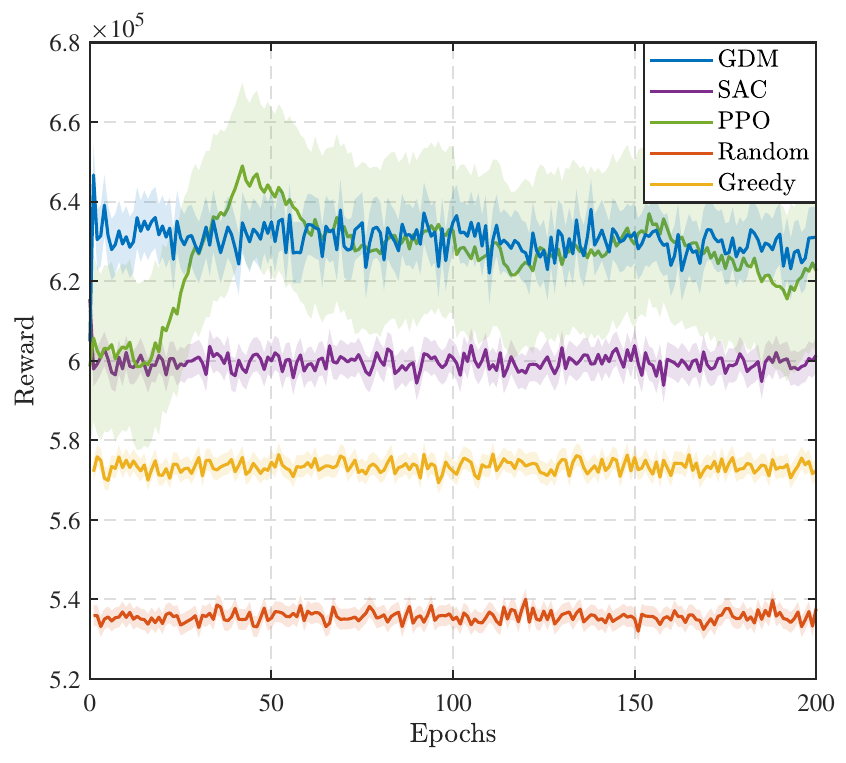}
\captionsetup{font=footnotesize} 
\caption{Reward comparison of our proposed GDM-based optimal contract design algorithm with other algorithms under PT, i.e., SAC, PPO, greedy, and random algorithms, with reference point $U_{ref}=10$ and loss aversion parameter $\kappa=0.5$.}
\label{performance}
\end{figure}

First, we demonstrate the convergence and superior performance of the proposed GDM algorithm. Figure \ref{performance} illustrates the performance improvements of various algorithms as the number of epochs increases. Notably, Figure \ref{performance} reveals that the GDM algorithm significantly outperforms both random and greedy algorithms. Furthermore,  Fig. \ref{performance} shows that the GDM algorithm consistently achieves the highest reward values, underscoring its superiority over traditional DRL algorithms, e.g., Proximal Policy Optimization (PPO) and Soft Actor Critic (SAC), under the same parameters settings. Although the reward value under the PPO algorithm is comparable to that of GDM, the GDM algorithm generally produces higher and more stable reward values. This result indicates the AV can obtain more utility by designing optimal contracts through the GDM algorithm, further emphasizing the effectiveness and robustness of the GDM algorithm. The impressive performance of GDM can be attributed to two key factors \cite{du2023deep}. On the one hand, fine-tuned policy adjustments during the diffusion process help mitigate the effects of randomness and noise. On the other hand, the exploratory nature of the diffusion process enhances the flexibility and robustness of the policy, reducing the likelihood of the model settling into suboptimal solutions. 

\begin{figure}[t]
\centering
\includegraphics[width=0.5\textwidth]{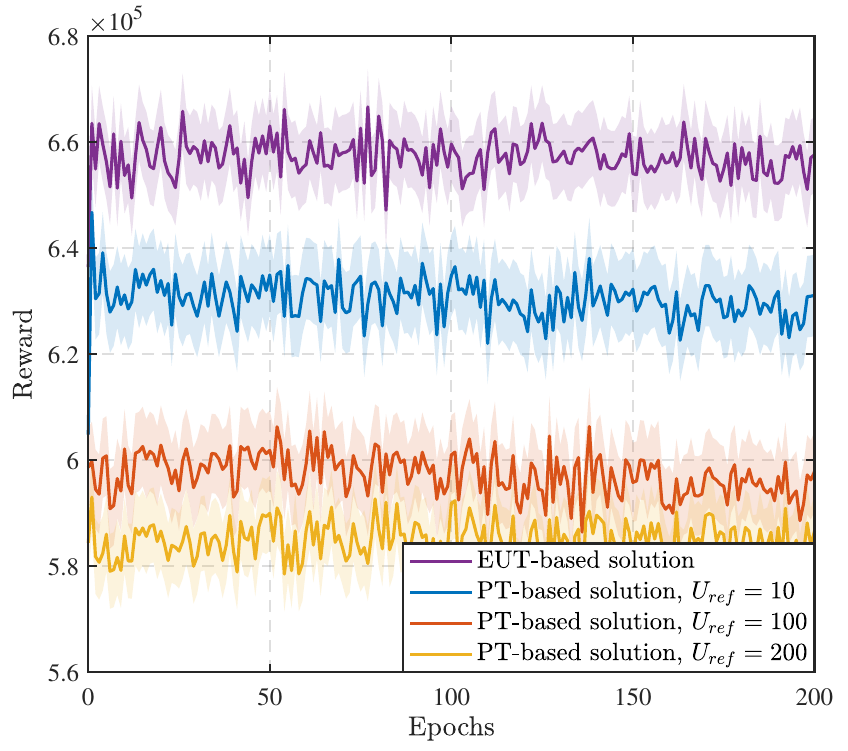}
\captionsetup{font=footnotesize}
\caption{Reward comparison of the proposed GDM-based optimal contract design algorithm under different reference points $U_{ref}$, with loss aversion parameter $\kappa = 0.5$.}
\label{U_ref}
\end{figure}
Figure \ref{U_ref} shows the trend of rewards over epochs for both PT-based and EUT-based solutions. The PT-based solution incorporates different reference points $U_{ref}$, accounting for PT in its contract design. Conversely, the EUT-based solution adheres to EUT and does not consider PT. From Fig. \ref{U_ref}, we can observe that the EUT-based solution consistently outperforms the PT-based solution in optimal contract design, regardless of the reference points $U_{ref}$. This is because the EUT-based solution does not consider the utility of AVs under uncertain and risky conditions. However, AVs may exhibit unreasonable behavior in this case. Thus the EUT-based solution is impractical. Furthermore, it is evident that as the reference point $U_{ref}$ increases, the corresponding reward diminishes, signifying a decrease in the utility of the AV. Therefore, it is shown that the smaller the reference point the AV is set, the greater the benefit can be obtained.

\begin{figure}[t]
\centering
\includegraphics[width=0.5\textwidth]{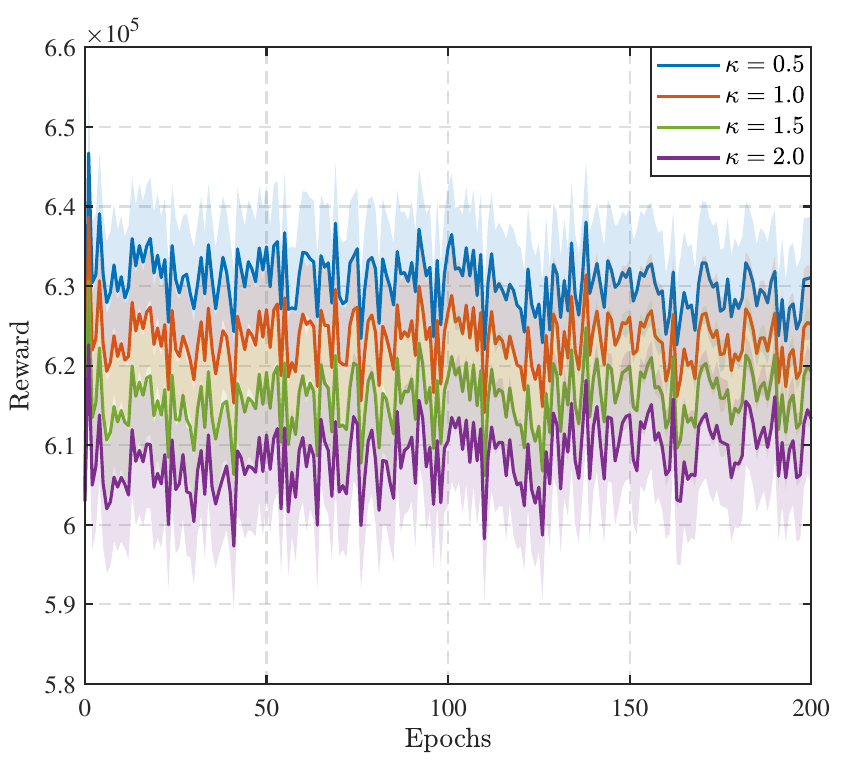}
\captionsetup{font=footnotesize}
\caption{Reward comparison of the proposed GDM-based optimal contract design algorithm under different loss aversion parameters $\kappa$, with reference point $U_{ref} = 10$.}
\label{aversion}
\end{figure}
Figure \ref{aversion} shows the rewards obtained under various loss aversion parameters $\kappa$, while keeping the reference point $U_{ref}$ constant. As shown in Fig. \ref{aversion}, we can observe that larger loss aversion results in smaller reward values. This occurs because an increase in the loss aversion parameter makes the AV more inclined to risk-averse behavior \cite{10254627}, indicating that AVs require RSUs with better resources to provide in-vehicle services and avoid utility loss. Consequently, AVs must offer higher rewards to RSUs with superior capabilities, thereby reducing the subjective utility of the AV and further lowering the overall rewards. Moreover, it is worth noting that the GDM algorithm consistently achieves optimal contracts with the same convergence speed, regardless of the variations in loss aversion parameters $\kappa$. This highlights the adaptability and robustness of the proposed scheme.

\begin{figure}[t]
\centering
\includegraphics[width=0.48\textwidth]{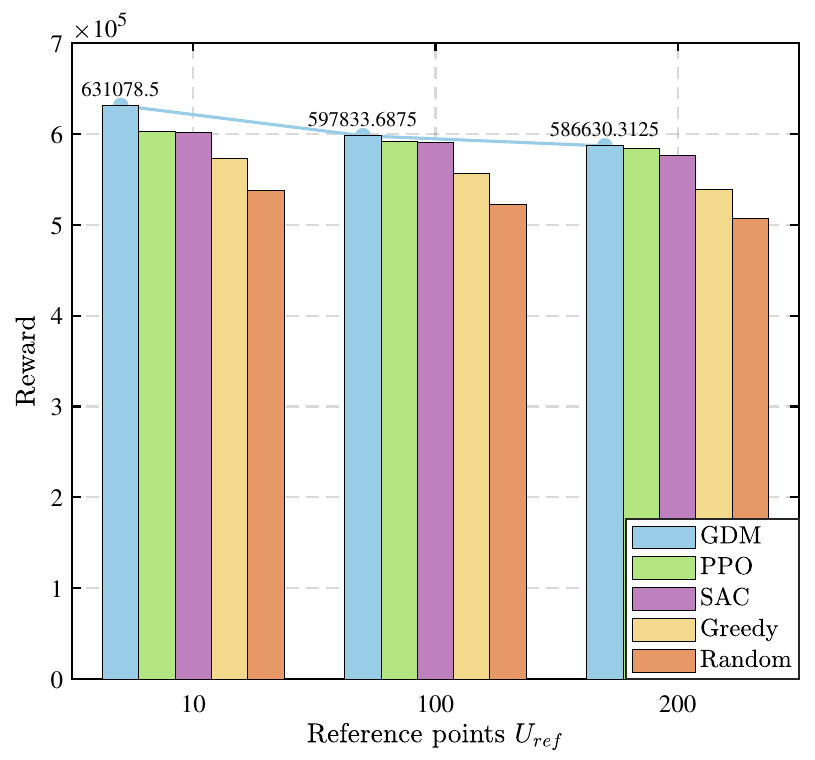}
\captionsetup{font=footnotesize}
\caption{Reward comparison of our proposed GDM-based optimal contract design algorithm with other algorithms under PT, under different reference points $U_{ref}$.}
\label{bar}
\end{figure}
Figure \ref{bar} illustrates the rewards achieved by different algorithms at various reference points $U_{ref}$. It is obvious from Fig. \ref{bar} that regardless of the algorithm used, e.g., GDM, SAC, PPO, greedy, and random algorithms, the reward will decrease as the reference point $U_{ref}$ set by the AV increases. This is because the larger the value of the reference point, the higher the requirements of the AV for RSUs, and therefore it will be more difficult to meet the needs of the AV, resulting in a decrease in the utility of the AV. In addition, it can be observed that the reward value under the GDM algorithm is always larger than the reward value of other algorithms, indicating that the GDM algorithm outperforms other algorithms, which further illustrates the superior performance of the proposed GDM algorithm.

\begin{figure}[t]
\centering
\includegraphics[width=0.5\textwidth]{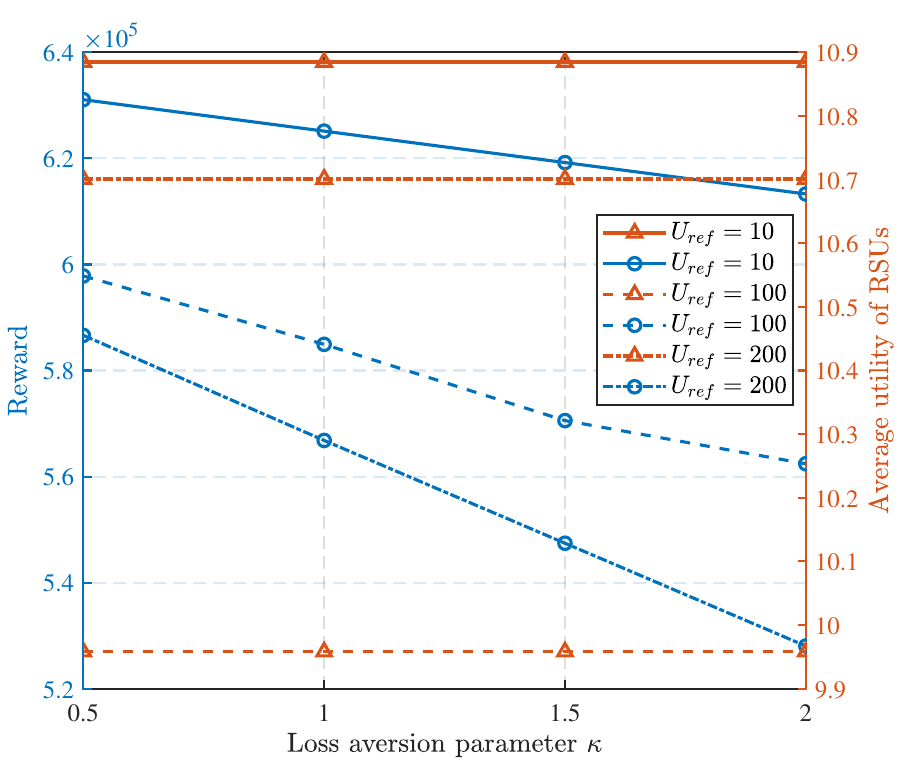}
\captionsetup{font=footnotesize}
\caption{Reward of the AV and average utility of RSUs comparison of our proposed GDM-based optimal contract design algorithm, under different reference points $U_{ref}$ and loss aversion parameters $\kappa$.}
\label{utility}
\end{figure}
Figure \ref{utility} depicts the reward of the AV and the average utility of RSUs in different preference parameters, i.e., reference point $U_{ref}$ and loss aversion parameter $\kappa$. From Fig. \ref{utility}, we observe that the reward decreases as the loss aversion parameter $\kappa$ increases, regardless of the reference point value $U_{ref}$. This finding corroborates the conclusion drawn in Fig. \ref{aversion}. In addition, it is observed that the reward decreases as the reference point $U_{ref}$ increases, irrespective of the loss aversion parameter $\kappa$, which supports the conclusion proposed in Fig. \ref{U_ref}. Furthermore, Fig. \ref{utility} demonstrates that the average utility of RSUs remains stable regardless of the loss aversion parameter $\kappa$, with the same reference point $U_{ref}$. This stability is attributed to the increase in objective utility for higher-type RSUs being offset by the decrease in objective utility for lower-type RSUs, resulting in a stable average utility for RSUs overall. 


\begin{figure}[t]
\centering
\subfigure[Contracts generated under different states.]{
\centering
\includegraphics[width=0.48\textwidth]{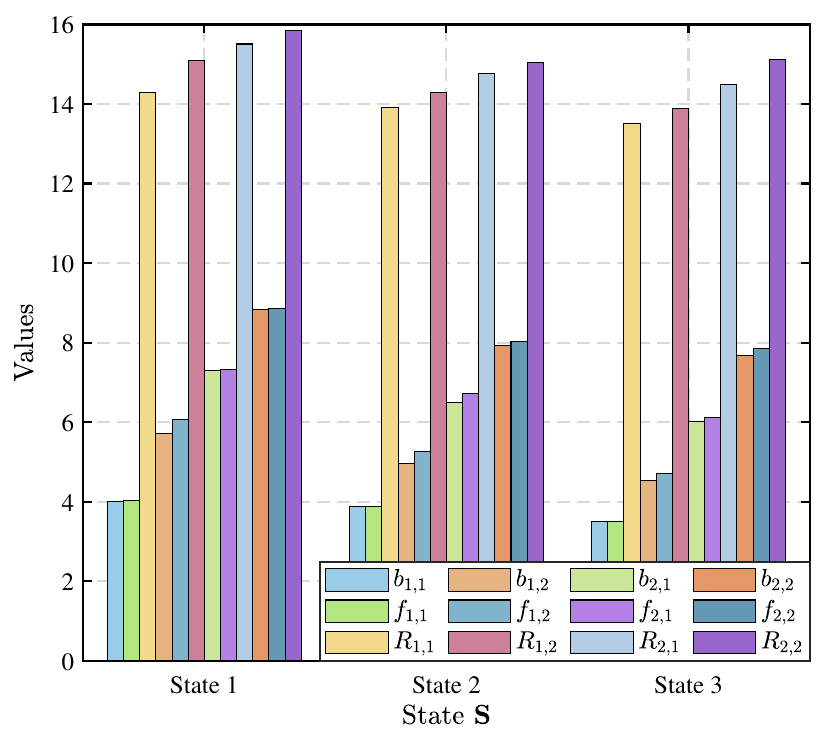}
\label{value}
}
\subfigure[Utilities of the AV and RSUs.]{
\centering
\includegraphics[width=0.5\textwidth]{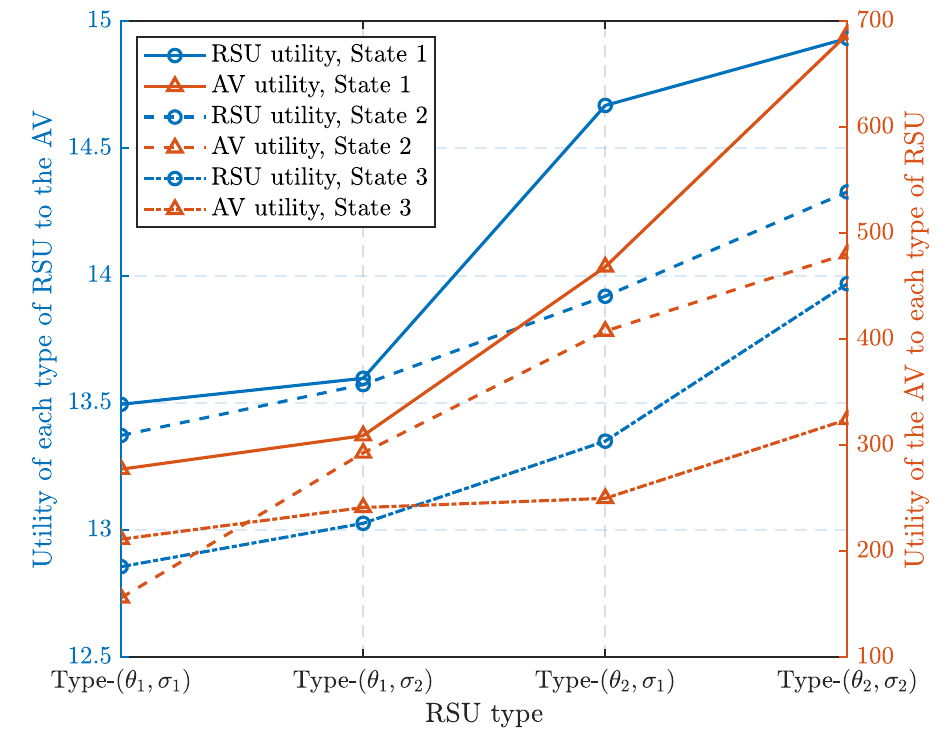}
\label{type}
}
\captionsetup{font=footnotesize}
\caption{Utilities of the AV and RSUs under different states $\boldsymbol{S}$.}
\label{contract_value}
\end{figure}
Figure \ref{contract_value} presents the utilities of the AV and RSUs across different states $\boldsymbol{S}$, showcasing the generation of various contracts. Figure \ref{value} illustrates the details of all types of contracts generated under different network states. We can observe that as the type of RSUs increases, the value of the contract also increases, which supports Lemma \ref{lemma1}. This further indicates that the AV will request more resources from RSUs with higher types. Moreover, from Fig. \ref{type}, it is evident that the utilities of both AV and RSUs exhibit a positive correlation with the increasing type of RSU. Specifically, we can also observe that $V_{2,2}^{2,2}\ge \max\{V_{1,2}^{1,2}, V_{2,1}^{2,1}, V_{1,1}^{1,1}\}$, i.e., the RSU with higher type will receive higher utility, confirming Lemma \ref{lemma2}. In summary, the simulation results demonstrate the feasibility of the contract model.

\section{Conclusion}\label{Conclusion}
In this paper, we introduced the concepts of ``embodied twins" and ``embodied AI twins" within the context of embodied AI. By integrating embodied AI with vehicular networks, a new paradigm called ``VEANETs" was proposed, where AVs play a crucial role. We tackled the challenge of efficient embodied AI twins migration in the networks by introducing a multi-dimensional contract model between AVs and RSUs. This model tackles the issue of information asymmetry, where AVs lack detailed knowledge about RSU resources. To account for the potential irrational behavior of AVs in risky and uncertain environments, we incorporated PT into the contract model. Specifically, PT is used to construct the utility function of AVs, allowing us to measure the subjective utility rather than the expected utility of AVs. Finally, we employed a GDM-based algorithm to determine the optimal contract design. Numerical results demonstrated the effectiveness and reliability of the proposed GDM-based contract design model under PT. For future work, we will focus on further refining the model to consider scenarios with multiple AVs and multiple RSUs.

\section*{Appendix A \\ Proof for Lemma \ref{lemma5}}\label{E}
There are $MN(MN-1)$ IC constraints defined in (\ref{IC}), which can be divided into $MN(MN-1)/2$ Downward Incentive Compatibility (DIC), shown as
\begin{equation}
    V_{m,n}^{m,n}\ge V_{m,n}^{i,j}, 1\le i\le M, 1\le j\le N, m>i,n>j,
\end{equation}
and $MN(MN-1)/2$ Upward Incentive Compatibility (UIC), shown as
\begin{equation}
    V_{m,n}^{m,n}\ge V_{m,n}^{i,j}, 1\le i\le M, 1\le j\le N, m<i,n<j.
\end{equation}

First, we prove the DIC can be reduced to LDIC. Based on the IC constraints, we can get $V_{m+1,n+1}^{m+1,n+1}\ge V_{m+1,n+1}^{m,n}$, which is the LDIC. 
Moreover, based on Lemma $\ref{lemma3}$, we can obtain $V_{m+1,n+1}^{m,n}\ge V_{m+1,n+1}^{m,n-1}$, $V_{m+1,n+1}^{m,n}\ge V_{m+1,n+1}^{m-1,n}$, and $V_{m+1,n+1}^{m,n}\ge V_{m+1,n+1}^{m-1,n-1}$. Considering the above analysis, we can get
\begin{equation}\label{vm}
    V_{m+1,n+1}^{m+1,n+1}\ge \max\{V_{m+1,n+1}^{m,n-1},V_{m+1,n+1}^{m-1,n},V_{m+1,n+1}^{m-1,n-1}\}.
\end{equation}
Therefore, we can know that the type-($\theta_{m+1},\sigma_{n+1}$) RSUs prefer to choose the contract item $\{b_{m+1,n+1},f_{m+1,n+1},R_{m+1,n+1}\}$ rather than contract item $\{b_{m,n-1},f_{m,n-1},R_{m,n-1}\}$, $\{b_{m-1,n},f_{m-1,n},R_{m-1,n}\}$ and $\{b_{m-1,n-1},f_{m-1,n-1},R_{m-1,n-1}\}$. It can be downward extended until type-($\theta_1,\sigma_1$) based on Eq. (\ref{vm}). Therefore, we can get
\begin{equation}\label{vm1}
    \begin{split}
        V_{m+1,n+1}^{m+1,n+1}\ge \max\{V_{m+1,n+1}^{m,n-1},V_{m+1,n+1}^{m-1,n},V_{m+1,n+1}^{m-1,n-1}\}\\
        \ge \cdots\ge \max\{V_{m+1,n+1}^{2,1},V_{m+1,n+1}^{1,2},V_{m+1,n+1}^{1,1}\}\\
        \ge\max\{V_{1,1}^{2,1},V_{1,1}^{1,2},V_{1,1}^{1,1}\}.
    \end{split}
\end{equation}
We can conclude that the DIC is upheld based on Lemma \ref{lemma3} and the LDIC. 
Additionally, $V_{n,m}^{m.n}\ge V_{n,m}^{m+1.n}$, i.e.,
\begin{equation}
    R_{m,n}-\frac{b_{m,n}^2}{\theta_m}-\frac{f_{m,n}^2}{\sigma_n}\ge R_{m,n+1}-\frac{b_{m,n+1}^2}{\theta_m}-\frac{f_{m,n+1}^2}{\sigma_n},
\end{equation}
i.e.,
\begin{equation}
    R_{m,n+1}-R_{m,n}-\frac{1}{\theta_m}\big(b_{m,n+1}^2-b_{m,n}^2\big)-\frac{1}{\sigma_n}\big(f_{m,n+1}^2-f_{m,n}^2\big)\le 0.
\end{equation}
Since $\theta_{m-1}<\theta_m$, $\sigma_{n-1}<\sigma_n$, $b_{m,n+1}-b_{m,n}>0$ and $f_{m,n+1}-f_{m,n}>0$, we can get
\begin{equation}
\begin{split}
    R_{m,n+1}-R_{m,n}-\frac{1}{\theta_{m-1}}\big(b_{m,n+1}^2-b_{m,n}^2\big)-\\
    \frac{1}{\sigma_{n-1}}\big(f_{m,n+1}^2-f_{m,n}^2\big)\le 0,
\end{split}
\end{equation}
which is equivalent to
\begin{equation}
    R_{m,n}-\frac{b_{m,n}^2}{\theta_{m-1}}-\frac{f_{m,n}^2}{\sigma_{n-1}}\ge R_{m,n+1}-\frac{b_{m,n+1}^2}{\theta_{m-1}}-\frac{f_{m,n+1}^2}{\sigma_{n-1}},
\end{equation}
i.e., $V_{m-1,n-1}^{m,n}\ge V_{m-1,n-1}^{m,n+1}$. Similarly, we can prove $V_{m-1,n-1}^{m,n}\ge V_{m-1,n-1}^{m+1,n}$ and $V_{m-1,n-1}^{m,n}\ge V_{m-1,n-1}^{m+1,n+1}$, i.e.,
\begin{equation}
    V_{m-1,n-1}^{m,n}\ge\max\{V_{m-1,n-1}^{m,n},V_{m-1,n-1}^{m+1,n},V_{m-1,n-1}^{m+1,n+1}\}.
\end{equation}
Since $V_{m-1,n-1}^{m-1,n-1}\ge V_{m-1,n-1}^{m,n}$, we can obtain
\begin{equation}
    V_{m-1,n-1}^{m-1,n-1}\ge\max\{V_{m-1,n-1}^{m,n+1},V_{m-1,n-1}^{m+1,n},V_{m-1,n-1}^{m+1,n+1}\}.
\end{equation}
Similarly to (\ref{vm1}), we can finally prove that if LUIC holds, then UIC holds.

\section*{Appendix B \\ Proof for Theorem \ref{theorem1}}\label{F}
Based on IC constraints, for $1\le m\le M$ and $  \:1\le n\le N$, we can get $V_{m-1,n}^{m-1,n}\ge V_{m-1,n}^{m-1,n-1}$, i.e.,
\begin{equation}
\begin{split}
     R_{m-1,n}&-\frac{b_{m-1,n}^2}{\theta_{m-1}}-\frac{f_{m-1,n}^2}{\sigma_n}\ge \\ 
     &R_{m-1,n-1}-\frac{b_{m-1,n-1}^2}{\theta_{m-1}}-\frac{f_{m-1,n-1}^2}{\sigma_n},
\end{split}
\end{equation}
which is equivalent to 
\begin{equation}
    \begin{split}
        &R_{m-1,n}-\frac{b_{m-1,n}^2}{\theta_m}-\frac{f_{m-1,n}^2}{\sigma_n}\ge \\
        &R_{m-1,n-1}-\frac{b_{m-1,n-1}^2}{\theta_{m-1}}-\frac{f_{m-1,n-1}^2}{\sigma_{n-1}}+\\
        &\bigg(\frac{1}{\theta_{m-1}}-\frac{1}{\theta_m}\bigg)\big(b_{m-1,n}^2-b_{m-1,n-1}^2\big)+\\
        &\bigg(\frac{1}{\sigma_{n-1}}-\frac{1}{\sigma_n}\bigg)f_{m-1,n-1}^2+\bigg(\frac{1}{\theta_{m-1}}-\frac{1}{\theta_m}\bigg)b_{m-1,n-1}^2,
    \end{split}
\end{equation}
that is
\begin{equation}\label{1}
    \begin{split}
        V_{m,n}^{m-1,n}&\ge V_{m-1,n-1}^{m-1,n-1}+\bigg(\frac{1}{\theta_{m-1}}-\frac{1}{\theta_m}\bigg)\big(b_{m-1,n}^2-\\
        &b_{m-1,n-1}^2\big)+\bigg(\frac{1}{\sigma_{n-1}}-\frac{1}{\sigma_n}\bigg)f_{m-1,n-1}^2+\\
        &\bigg(\frac{1}{\theta_{m-1}}-\frac{1}{\theta_m}\bigg)b_{m-1,n-1}^2.
    \end{split}
\end{equation}
Similarly,  $V_{m,n-1}^{m,n-1}\ge V_{m,n-1}^{m-1,n-1}$, i.e.,
\begin{equation}
    \begin{split}
     R_{m,n-1}&-\frac{b_{m,n-1}^2}{\theta_{m}}-\frac{f_{m,n-1}^2}{\sigma_{n-1}}\ge \\ 
     &R_{m-1,n-1}-\frac{b_{m-1,n-1}^2}{\theta_m}-\frac{f_{m-1,n-1}^2}{\sigma_{n-1}},
     \end{split}
\end{equation}
which is equivalent to
\begin{equation}\label{2}
    \begin{split}
        V_{m,n}^{m,n-1}&\ge V_{m-1,n-1}^{m-1,n-1}+\bigg(\frac{1}{\theta_{m-1}}-\frac{1}{\theta_m}\bigg)b_{m-1,n}^2+\\
        &\bigg(\frac{1}{\sigma_{n-1}}-\frac{1}{\sigma_n}\bigg)f_{m-1,n-1}^2+\bigg(\frac{1}{\sigma_{n-1}}-\frac{1}{\sigma_n}\bigg)\\
        &\big(f_{m,n-1}^2-f_{m-1,n-1}^2\big).
    \end{split}
\end{equation}
According to IC constraints, we can get
\begin{equation}\label{a}
    V_{m,n}^{m,n}\ge \max\{V_{m,n}^{m,n-1},V_{m,n}^{m-1,n},V_{m,n}^{m-1,n-1}\}.
\end{equation}
The AV will minimize the reward to optimize profit until the equal sign of Eq.(\ref{a}) is satisfied. Thus considering Eqs. (\ref{1}) and (\ref{2}), we have the recurrence formula as
\begin{equation}\label{a}
    \begin{split}
        V_{m,n}^{m,n}=&V_{m-1,n-1}^{m-1,n-1}+\bigg(\frac{1}{\theta_{m-1}}-\frac{1}{\theta_m}\bigg)b_{m-1,n-1}^2+\\
        &\bigg(\frac{1}{\sigma_{n-1}}-\frac{1}{\sigma_n}\bigg)f_{m-1,n-1}^2+\\
        &\max\bigg\{0,\bigg(\frac{1}{\theta_{m-1}}-\frac{1}{\theta_m}\bigg)\big(b_{m-1,n}^2-b_{m-1,n-1}^2\big),\\
        &\:\:\quad\bigg(\frac{1}{\sigma_{n-1}}-\frac{1}{\sigma_n}\bigg)\big(f_{m,n-1}^2-f_{m-1,n-1}^2\big)\bigg\}.
    \end{split}
\end{equation}
Through iterative calculation, we can get the expression of $V_{m-1,n-1}^{m-1,n-1}$ with respect to $V_{m-2,n-2}^{m-2,n-2}$, and the expression of $V_{m-2,n-2}^{m-2,n-2}$ with respect to $V_{m-3,n-3}^{m-3,n-3}$, and by analogy, Eq. (\ref{a}) can be formulated as
\begin{equation}\label{reduce2}
\begin{split}
    V_{m,n}^{m,n}&=V_{1,1}^{1,1}+\sum_{i=1}^{m-1}\sum_{j=1}^{n-1}\big(\Delta_ib_{i,j}^2+\Lambda_jf_{i,j}^2\big)+\sum_{i=1}^{m-1}\sum_{j=1}^{n-1}\\
    &\max\bigg\{0,\Delta_i\big(b_{i,j+1}^2-b_{i,j}^2\big), \Lambda_j\big(f_{i+1,j}^2-f_{i,j}^2\big)\bigg\},
\end{split}
\end{equation}
where $\Delta_i=\frac{1}{\theta_i}-\frac{1}{\theta_{i+1}}>0$, and $\Lambda_j=\frac{1}{\sigma_j}-\frac{1}{\sigma_{j+1}}>0$. For the reduced IR constraint $V_{1,1}^{1,1}>0$ derived in Lemma \ref{lemma4}, the AV will reduce the reward as much as possible to maximize its objective function until $V_{1,1}^{1,1}=0$ \cite{ho2020denoising}. Thus, Eq. (\ref{reduce2}) can be formulated as Eq. (\ref{reduce1}).

\bibliographystyle{IEEEtran}
\bibliography{ref}

\end{document}